\documentclass[10pt]{article} 
\usepackage[preprint]{tmlr}

\usepackage{amsmath,amsfonts,bm}

\def\1{\bm{1}}
\newcommand{\train}{\mathcal{D}}
\newcommand{\valid}{\mathcal{D_{\mathrm{valid}}}}
\newcommand{\test}{\mathcal{D_{\mathrm{test}}}}

\DeclareMathAlphabet{\mathsfit}{\encodingdefault}{\sfdefault}{m}{sl}
\SetMathAlphabet{\mathsfit}{bold}{\encodingdefault}{\sfdefault}{bx}{n}

\newcommand{\E}{\mathbb{E}}

\newcommand{\R}{\mathbb{R}}

\newcommand{\KL}{D_{\mathrm{KL}}}

\DeclareMathOperator*{\argmax}{arg\,max}
\DeclareMathOperator*{\argmin}{arg\,min}

\usepackage{hyperref}
\usepackage{url}

\usepackage{amsmath}
\usepackage{amssymb}
\usepackage{mathtools}
\usepackage{amsthm}
\usepackage{dsfont}				
\usepackage{multirow}
\usepackage{subcaption}
\usepackage{graphicx}
\usepackage{booktabs} 
\usepackage[ruled]{algorithm2e}
\usepackage{bm}
\usepackage{physics}
\usepackage[capitalize,noabbrev]{cleveref}

\usepackage{color}

\newcommand{\param}{\theta}       
\newcommand{\dimparam}{d}         
\newcommand{\setparam}{\Theta}    
\newcommand{\T}{T}            
\newcommand{\h}{h}            
\newcommand{\m}{m}            
\newcommand{\ntrain}{n}						 
\newcommand{\Dbatch}{\mathcal{B}}            
\newcommand{\LPD}{\operatorname{LPD}}
\newtheorem{lemma}{Lemma}[section]
\DeclarePairedDelimiterX{\infdivx}[2]{(}{)}{%
  #1\delimsize\|#2%
  }

\title{Temperature Optimization for Bayesian Deep Learning}

\author{\name Kenyon Ng \thanks{Main part of the work
    done while at the University of Melbourne.}
  \email kenyon.ng@monash.edu \\
  \addr Department of Econometrics and Business Statistics \\
  Monash University \AND
  \name Chris van der Heide \email chris.vanderheide@unimelb.edu.au \\
  \addr Department of Electrical and Electronic Engineering\\
  University of Melbourne \AND
  \name Liam Hodgkinson \email lhodgkinson@unimelb.edu.au\\
  \addr School of Mathematics and Statistics \\
  University of Melbourne \AND
  \name Susan Wei \footnotemark[1]
  \email susan.wei@monash.edu\\
  \addr Department of Econometrics and Business Statistics \\
  Monash University}

\def\month{MM}  
\def\year{YYYY} 
\def\openreview{\url{https://openreview.net/forum?id=XXXX}} 

\begin{document}

\maketitle

\begin{abstract}
  The Cold Posterior Effect (CPE) is a phenomenon in Bayesian Deep Learning
  (BDL), where tempering the posterior to a cold temperature often improves the
  predictive performance of the posterior predictive distribution (PPD).
  Although the term `CPE' suggests colder temperatures are inherently better,
  the BDL community increasingly recognizes that this is not always the case.
  Despite this, there remains no systematic method for finding the optimal
  temperature beyond grid search. In this work, we propose a data-driven
  approach to select the temperature that maximizes test log-predictive density,
  treating the temperature as a model parameter and estimating it directly from
  the data. We empirically demonstrate that our method performs comparably to
  grid search, at a fraction of the cost, across both regression and
  classification tasks. Finally, we highlight the differing perspectives on CPE
  between the BDL and Generalized Bayes communities: while the former primarily
  emphasizes the predictive performance of the PPD, the latter prioritizes the
  utility of the posterior under model misspecification; these distinct
  objectives lead to different temperature preferences.
\end{abstract}

\section{Introduction}
Recent decades have seen substantial advances in machine learning, with
deployments across diverse applications. These range from image recognition and
natural language processing to autonomous vehicles and medical diagnostics.
Driven in no small part by the technological advances that enables large-scale
training of high-fidelity models, these models obtain impressive generalization
performance that are not readily explained with classical statistical wisdom.

One promising avenue that has been explored for explaining and enhancing
performance, as well as providing much-needed robustness guarantees, is
\textbf{Bayesian deep learning} (BDL). However, extending classical Bayesian
techniques to the machine learning setting is a formidable task---the scale of
both the data and model size render na\"ive extensions intractable. Combined with
the fact that neural network models are typically singular \citep{wei22deep},
the mismatch between the size of the data and the number of parameters makes
many of the classical asymptotics inappropriate when applied to problems of
interest. Furthermore, a number of phenomena have been observed when analyzing
deep learning models that were previously considered curiosities or corner-cases
in the classical setting.

A prominent example of this is the so-called \textbf{cold posterior effect}
(CPE). Named in analogy with statistical physics, tempered posteriors are
obtained by raising the posterior density to the power of an artificial
inverse-temperature parameter. The CPE refers to improved generalization
performance of the \textbf{posterior predictive density} (PPD) in both
regression \citep{adlam20cold} and classification \citep{wenzel20how} tasks when
the temperature $\T$ is taken to be cold with $0< \T < 1$. This peculiar effect
is frequently viewed as a hack to improve generalization performance, and has
led to a plethora of works that attempt to explain or `fix' the CPE, see
\cref{sec:bdl-lens} for a literature review.

\paragraph{Contribution.}
There is growing recognition in the BDL community that, despite the term `CPE',
colder temperatures do not always result in better predictive performance for
the PPD \citep{adlam20cold,zhang24cold}. Unfortunately, the common approach of
temperature tuning via grid search is computationally expensive, as it requires
extra posterior sampling for multiple $\T$ across the grid. To address these
issues, we propose a data-driven method to select an appropriate temperature. To
the best of our knowledge, no such dedicated tool exists without appealing to
intermediate approximations (e.g.,~variational inference
\citep{laves21posterior}). Our method only requires maximizing a likelihood
function to find a suitable temperature. This can usually be done as part of the
sampler warm-up phase, and importantly, does not require any extra posterior
sampling.

\section{Tempered Posteriors and the Cold Posterior Effect}
\label{sec:poss-expl-cpe}
Let $\{q(x, y), p(x, y | \param), p(\param)\}$ form a triplet representing the
truth-model-prior, where the data-generating mechanism is
$q(x, y) = q(y | x) q(x)$, the model $p(x, y | \param) = p(y | x, \param) q(x)$
is indexed by $\param \in \setparam \subseteq \R^{\dimparam}$ representing
neural network weights, and the prior on $\theta$ is $p(\theta)$. Here, we focus
on a supervised learning setup with a training dataset
$\train = \{(x_i, y_i)\}_{i=1}^\ntrain$ containing $\ntrain$ observations drawn
from $q(x,y)$. The standard Bayesian update is derived from Bayes' theorem,
resulting in the \textit{standard posterior} distribution,
\begin{equation}
  \label{eq:standard-rule}
  p(\param | \train) \propto  p(\param) \prod_{(x, y) \in \train } p(y | x, \param).
\end{equation}
By introducing a `temperature' parameter, we obtain a family of \textit{tempered
  posteriors} $p_{\beta}(\param | \train)$ by raising the likelihood and prior
to the power of the \textit{inverse temperature} $\beta \coloneq \frac{1}{T}$
for $\beta>0$,
\begin{equation}
  \label{eq:temper-posterior}
  p_{\beta}(\param | \train)
  \propto  p(\param)^{\beta}
  \prod_{(x, y) \in \train} p(y | x, \param)^{\beta}.
\end{equation}
The CPE describes a phenomenon in Bayesian deep learning where the PPD,
\begin{equation}
  \label{eq:smpd}
  p_{\beta}(y | x, \train)
  = \int p(y | x, \param) p_{\beta}( \param| \train ) \, \textrm{d}\param,
\end{equation}
which is constructed from Bayesian model averaging, can achieve better
performance (in terms of the test log predictive density defined below) in
regression and classification by artificially tempering the posterior to
$\beta > 1$ \citep{wenzel20how}. While our construction in
\eqref{eq:temper-posterior} follows the convention of tempering both the
likelihood \textit{and} the prior \citep{wenzel20how, fortuin22bayesian},
improvements in performance have also been observed when only tempering the
likelihood \citep{aitchison21statistical,bachmann22how, kapoor22uncertainty}.

Define the test log predictive density (LPD) of a predictive density
$\hat p(y|x)$ as
\begin{equation*}
  \LPD(\hat p(y|x)) \coloneq \E_{q(x,y)} \log \hat p(y|x),
\end{equation*}
where the hat over $p$ indicates dependence on the training data $\train$, and
the expectation is only taken over the `new' observation $(x,y)$. Note that the
test~LPD is sometimes referred to as the `test log-likelihood' in the
literature, and the negative test LPD is often called `negative log-likelihood
(NLL)'. However, we will avoid these terms as they can sound ambiguous. It can
be shown that a predictive density $\hat p(y|x)$ with a higher test~LPD is
closer to the truth $q(y|x)$ in Kullback-Leibler (KL) divergence
$\E_{q(x)} \KL(q(y | x) \| \hat p(y|x))$. Our primary interest in this
work is to select $\beta$ such that the test LPD of \eqref{eq:smpd} is high.

\subsection{Outline}
Following the reasoning that the aleatoric uncertainty in the data can be
quantified by $\beta$ \citep{adlam20cold,kapoor22uncertainty}, we advocate a
likelihood-based approach to select $\beta$ directly from the data; details are
presented in \cref{sec:how-pick-right}. We empirically show that our method can
select a $\beta$ that is near-optimal in terms of test LPD. Experimental
results are presented in \cref{sec:experiments}.

We then conclude with a detailed discussion to dispel some misconceptions
surrounding the CPE. In particular, we review recent work in BDL in
\cref{sec:bdl-lens}, and explains why colder temperatures do not give better
test LPD. We also address the converse hypothesis sometimes encountered in
Generalized Bayes (GB) --- that warmer temperatures are often better --- in
\cref{sec:model-free-bayesian}. This hypothesis may hold in certain contexts,
e.g.,~ensuring posterior consistency \citep{grunwald12safe}, but does not
usually improve test LPD.

We emphasize that these two communities prioritize different objectives: in BDL,
the primary focus is on the PPD $p(y | x, \train)$ (e.g., reliable uncertainty
quantification of $y$), while in GB, the emphasis lies on the utility of the
posterior $p(\param | \train)$ (e.g., fast concentration rate as
$\ntrain \to \infty$, or credible intervals of $\param$ with good coverage).
Furthermore, the types of statistical models and datasets explored in each
community are notably distinct. In BDL, large-scale neural networks are trained
on extensive, nearly noiseless datasets, whereas in GB, the datasets are
typically smaller, noisier, and the models are classical regular statistical
models. Consequently, it is not surprising that each community arrives at
different recommendations regarding temperature tuning.

\section{Temperature Selection for Test LPD}
\label{sec:how-pick-right}
It is commonly believed that the PPD \eqref{eq:smpd} at $\beta=1$ is optimal for
test LPD when the model is well-specified
\citep{adlam20cold,aitchison21statistical}. However, this is not always the case
(see \cref{sec:div-truth-ppd} for a counter-example). In fact, in the likelihood
tempering case, \citet[Theorem 4]{zhang24cold} shows that $\beta = 1$ can only
be optimal if and only if the training loss,
$\sum_{(x,y)\in \train} \log p(y| x, \train)$ remains unchanged with the
inclusion of new data. We arrived at a similar conclusion for posterior
tempering; see \cref{sec:optimality-test-lpd} for details. This motivates the
development of an efficient method to select the optimal $\beta$.

\subsection{On the Role of Temperature in PPD via Singular Learning Theory}
If we could have some theoretical grasp on how temperature affects the test LPD
of the PPD in~\eqref{eq:smpd}, that is, $\LPD(p_{\beta}(y | x, \train))$, then
it might suggest a methodology for temperature selection. A little known result
to the BDL and GB communities is the following from singular learning theory
\citep[Lemma 3]{watanabe10asymptotic}, which applies to both regular and
singular models $p(y|x,\param)$ \footnote{A model $p(y | x, \param)$ is
  \emph{regular} if the corresponding Fisher information matrix,
  $\E_{q(x,y)}[\nabla_{\param} \log p(y|x, \param)(\nabla_{\param} \log p(y|x, \param))^\top]$,
  is positive-definite. Otherwise, the model is \emph{singular}.}:
\begin{align}
   & -\E_{\train} \LPD(p_{\beta}(y | x, \train)) \nonumber                                                                                      \\
   & \quad = -\E_{\train} \E_{q(x, y)} \log p_{\beta}(y | x, \train) \nonumber                                                                  \\
   & \quad = -\E_{q(x, y)} \log p(y|x, \param_\dag) + \left [\frac{\lambda-\nu(\beta)}{\beta} + \nu(\beta) \right ] \frac{1}{\ntrain} \nonumber \\
   & \qquad + o\left( \frac{1}{\ntrain}\right)
  \label{eq:slt}
\end{align}
where $\param_\dag \in \argmin_{\param} \KL(q(y|x) \| p(y|x,\param))$. Here
$\lambda$ and $\nu$ are strictly positive numbers, respectively known as the
learning coefficient and singular fluctuation. They are invariants of the
underlying truth-model-prior triplet. Note that $\lambda$ is independent of $\beta$
while $\nu$ is a (complex) function of $\beta$. The functional dependence of $\nu$
on $\beta$ is unknown in the current singular learning theory literature.

Some comments on \eqref{eq:slt} are in order. This result allows for both
misspecification ($\param_\dag$ is not necessarily such that the KL divergence
between the truth and $p(y|x,\param_\dag)$ is zero) and singular models
$p(y|x,\param)$, including neural networks, as well as classical models
satisfying standard regularity conditions. On the other hand, it is an
asymptotic result in the sample size $\ntrain$, and hence the prior plays no
role. We also remark that the relation in \eqref{eq:slt} pertains to the
\textit{average} (negative) test LPD of the PPD in \eqref{eq:smpd}, with the
average taken over the training set $\train$, as indicated by the notation
$\E_{\train}$. More precise interpretation of \eqref{eq:slt} can be divided into
two settings:

\paragraph{Regular models.} Let $\dimparam$ be the dimension of $\param$. For
well-specified regular models, $\lambda = \nu(\beta) = \dimparam / 2$ for all
$\beta$ \citep{watanabe09algebraic}, and the second term,
$\left [ \frac{\lambda-\nu(\beta)}{\beta} + \nu(\beta) \right ] \frac{1}{\ntrain}$,
reduces to $\frac{\dimparam}{2\ntrain}$. This implies that for large~$\ntrain$
relative to $\dimparam$, temperature has little impact (though it may appear in
higher-order terms in the expansion). This aligns with recent work by
\citet{mclatchie24predictive}, which arrives at a similar conclusion using
different techniques. Under model misspecification, the second term can
alternatively be expressed as $\frac{\beta}{n} \E_\train V(n)$
\citep{watanabe10equations}, where $V(n)$ represents the functional variance and
now depends on the temperature in a non-trivial manner.

\paragraph{Singular models.} The asymptotic relation in \eqref{eq:slt} also
applies to singular models, e.g.,~neural network models \citep{wei22deep}.
However, the theoretical values of $\lambda$ are generally unknown, except for a
few simple models, such as one-layer tanh or reduced rank regression, in the
well-specification setting
\citep{yamazaki03singularities,aoyagi05stochastic,rusakov05asymptotic,zwiernik11asymptotic}.
Similarly, the singular fluctuation and its temperature dependence are unknown,
even in the well-specification setting.

Since our setup involves modern deep learning models, we are dealing with
singular models, where $\lambda$ and $\nu$ are unknown. While methods to
estimate $\lambda$ and $\nu$ from training data do exist
\citep{lau23quantifying,watanabe10asymptotic}, they require posterior sampling
over neural network weights. In principle, we could use these sample-based
estimates to select the optimal $\beta$, but this would require posterior
sampling at multiple temperatures and is no better than a grid search. As this
approach is computationally challenging for deep learning applications, we are
motivated to propose a data-driven technique for efficiently determining the
optimal temperature, which we detail in the next section.

\subsection{Selecting Temperature Using the Tempered Model}
\label{sec:model-prior-induced}
In this section, we introduce a method to select $\beta$ that corresponds to
high test LPD in both regression and classification tasks. Our approach is
grounded in the insight that the tempered posterior can be reverse-engineered as
the posterior from an alternative model-prior pair
\citep{zeno20why,zhang24cold}. Following these works, we define the
\textit{tempered model},
\begin{equation}
  \label{eq:tempered-model}
  p(y | x, \param, \beta)
  \coloneq \frac{p(y | x, \param)^{\beta}}
  {\int p(y^{\prime} | x, \param)^{\beta} \textrm{d}y^{\prime} } .
\end{equation}
From here, the tempered posterior in~\eqref{eq:temper-posterior} can be equivalently expressed
as
\begin{equation}
  p_{\beta}(\param | \train)
  \propto \tilde p (\param | \beta)
  \prod_{(x, y) \in \train} p(y|x, \param, \beta),
  \label{eq:reformulated-tempered-posterior}
\end{equation}
where the `rest of the terms',
$\tilde p (\param | \beta) \propto p(\param)^{\beta} \prod_{x \in \train} \int p(y^{\prime} | x, \param)^{\beta} \textrm{d}y^{\prime}$,
can be viewed as a prior on $\param$ with the `normalizing constant' being a
function of $x$ in $\train$. The term $\tilde p (\param | \beta)$ can be seen as
an input-dependent prior \citep{zeno20why} but we suppress the dependence on $x$
in the notation. The (inverse) temperature $\beta$ has an intuitive
interpretation of controlling the `spikiness' of the tempered model. We give two
examples with Gaussian (regression) and categorical (classification) models:
\paragraph{Regression.}
Given an arbitrary scalar function $\mu(x; \param)$ and a fixed, known variance
$\sigma^2$, we have a Gaussian model:
$p(y | x, \param) = \mathcal{N}(y | \mu(x; \param), \sigma^2)$. This leads to
\begin{gather*}
  p(y | x, \param, \beta)
  = \mathcal{N}\left(y \left|\, \mu(x; \param), \frac{\sigma^2}{\beta}\right.\right), \\
  \tilde p(\param | \beta)
  \propto p(\param)^{\beta} \left( \dfrac{2\pi \sigma^2}{\beta(2\pi \sigma^2)^{\beta}} \right)^{\ntrain/2} .
\end{gather*}
Therefore, the temperature effectively scales the model and prior variance. For
brevity, we suppress the dependency of $p(y | x, \param, \beta)$ on the
fixed $\sigma^{2}$ in the notation.

\paragraph{Classification.}
In $K$-class classification, we have
$p(y | x, \param) = f_y(x; \param)$ for $ y \in \{1,\ldots, K\}$, where
$f_y(\cdot)$ denotes the $y$-th entry of a softmax output $f$.
This leads to
\begin{gather*}
  p(y | x, \param, \beta)
  = \dfrac{\exp (\beta f_{y}(x; \param) )}{\sum^{K}_{k=1} \exp (\beta f_{k}(x; \param))}, \\
  \tilde p(\param | \beta)
  \propto p(\param)^{\beta} \prod_{x \in \train} \sum_k [f_{k}(x; \param)]^{\beta} .
\end{gather*}
This tempered model is also known as the \textit{tempered softmax}
\citep{hinton15distilling, agarwala23temperature}. For large $\beta$, the
tempered model will concentrate most of the mass in one class, and the converse
will encourage a more uniform distribution of mass across all classes. For
$\beta > 1$, the prior $\tilde p(\param | \beta)$ will also favor~$f$ that
concentrates mass in one class.

As $\beta$ can be used to capture aleatoric uncertainty in the data
\citep{adlam20cold,kapoor22uncertainty}, we propose selecting $\beta$ using a
maximum likelihood estimator for the tempered model in
\eqref{eq:tempered-model}:
\begin{equation}
  \label{eq:temp-mle}
  \hat \param^{*}, \hat \beta^{*}
  \coloneq \argmax_{\param, \beta} \frac{1}{\ntrain} \sum_{(x,y) \in \train}[\log p(y | x, \param, \beta )].
\end{equation}
Standard consistency results imply that, under mild conditions, provided the
model is regular and \emph{some} tempered model is well-specified, choosing
parameters according to \eqref{eq:temp-mle} recovers the optimal model as
$\ntrain \to \infty$ \citep[Chapter 5.2]{vandervaart98asymptotic}. In practice, we
optimize this using SGD, and stop when the log-likelihood shows no further
improvement on a validation set. To ensure that $\beta$ remains positive, we
reparameterize it as $\exp(\log \beta)$ and optimize with respect to
$\log \beta$. Further implementation details are provided in
\cref{sec:algorithm-structure}. We also discussed several variants of our
proposed method in \cref{sec:alternative-strategies}.

A similar temperature optimization approach was proposed in
\citet{guo17calibration} for computing a well-calibrated `plug-in' predictive
density $p(y | x, \param^{*}_{\text{SGD}}, \beta)$, where
$\param^{*}_{\text{SGD}}$ is typically an SGD solution of neural networks from a
standard training workflow. In their method, the optimal $\beta$ is computed
\emph{post hoc} by maximizing
$\frac{1}{\ntrain} \sum_{(x,y) \in \valid} p(y | x, \param^{*}_{\text{SGD}}, \beta)$
on a validation set, with $\param$ fixed and only $\beta$ being optimized. While
the difference appears subtle, we find that jointly optimizing $\param$ and
$\beta$ is the key to a good $\beta$ for constructing PPDs. Detailed
experimental results are provided in \cref{sec:alternative-strategies}.

\subsection{Supporting Theory}
\label{sec:supporting-theory}
Ideally, we would like to theoretically show that the temperature selection
method in \eqref{eq:temp-mle} produces high test LPD of \eqref{eq:smpd}.
However, it turns out that the theoretical guarantees are more natural for a
related object which is suggested by the reformulation of the tempered posterior
in \eqref{eq:reformulated-tempered-posterior}. Namely, we can consider an
alternative PPD that is the expectation of the \textit{tempered model}:
\begin{equation}
  \label{eq:tmpd}
  \E_{p_{\beta}(\param | \train)} p(y | x, \param, \beta)
  = \int p(y | x, \param, \beta) p_{\beta}(\param | \train) \textrm{d}\param.
\end{equation}
This is also the object of study in \citet{adlam20cold} and is to be contrasted
with the PPD in~\eqref{eq:smpd}. We emphasize that our temperature selection
method can be used with either \eqref{eq:smpd} or \eqref{eq:tmpd}, and we will
compare the resulting performance of our method for these two PPDs in the experiments in
\cref{sec:experiments}.

To study the theoretical properties of the temperature selection method for
\eqref{eq:tmpd}, we consider the objective at the population level, leading to
\begin{equation*}
  \param^{*}, \beta^{*} \coloneq \argmax_{\param, \beta} \E_{q(x, y)} \log p(y | x, \param, \beta ).
\end{equation*}
We justify $\beta^*$ in the case of Gaussian linear regression, but our
temperature methodology can be applied in far more general settings as we
demonstrate in \cref{sec:experiments}. We first compute the test LPD constructed
with~\eqref{eq:tmpd} and show that $\beta^{*}$ approximately maximizes a lower
bound of the test LPD.

\begin{lemma}
  \label{thm:linear-lpd}
  Consider a linear regression model
  $p(y | x, \param) = \mathcal{N}(y | x^{\top} \param, \sigma^{2})$ with a
  $d$-dimensional input $x$ and known variance $\sigma^{2}$, and a prior
  $p(\param) = \mathcal{N}(\param | 0, \sigma^{2}_{p})$ with finite variance
  $\sigma^{2}_{p}$. Let
  $\bm{X} \coloneq (x_{1}, \ldots, x_{\ntrain})^{\top} \in \mathbb{R}^{\ntrain \times d}$
  and
  $\bm{\Sigma} \coloneq (\bm{X}^{\top}\bm{X} + \frac{\sigma^{2}}{\sigma^{2}_{p}} \bm{I})^{-1}$.
  The test LPD of the PPD in~\eqref{eq:tmpd} at a fixed $\beta$ is bounded from
  below:
  \begin{align*}
     & \LPD(\E_{p_{\beta}(\param | \train)}[p(y | x, \param, \beta)])          \\
     & \quad > \E_{q(x, y)}  \log p(y | x, \hat{\param}_{\textrm{MAP}}, \beta)
    -\frac{1}{2} \E_{q(x, y)}  \log{(1 + x^{\top} \bm{\Sigma} x)} ,
  \end{align*}
  where $\hat{\param}_{\textrm{MAP}} \coloneq \bm{\Sigma} \bm{X}^{\top} \bm{y}$
  is the \emph{maximum-a-posteriori} solution of the posterior
  $p_{\beta}(\param | \train)$ at $\beta = 1$ and
  $\bm{y} \coloneq (y_{1}, \ldots y_{\ntrain})^{\top} \in \mathbb{R}^{\ntrain}$.
\end{lemma}

The proof can be found in \cref{sec:proof-linear-lpd}. As our goal is to select
$\beta$ that maximizes test LPD, a reasonable strategy is to optimize the lower
bound presented in \cref{thm:linear-lpd} with respect to $\beta$. As the second
term in the lower bound is independent of $\beta$, this is therefore equivalent
to maximizing the first term,
$\E_{q(x, y)} \log p(y | x, \hat{\param}_{\text{MAP}}, \beta)$. However, this
objective function requires $\hat{\param}_{\text{MAP}}$, which is often
unavailable in closed-form. A straightforward solution is to replace it with an
estimate from an iterative optimizer, before optimizing again with respect to
$\beta$. This comes at the cost of two optimization runs (and the effort to tune
their hyperparameters). Instead, we propose maximizing the empirical version of
$\E_{q(x, y)} \log p(y | x, \param, \beta)$ with respect to $\param$ and $\beta$
simultaneously, leading us back to \eqref{eq:temp-mle}. We study the efficacy of
this method empirically in the next section.

\section{Experiments}
\label{sec:experiments}
We now illustrate the behavior of the PPDs, \eqref{eq:smpd} and \eqref{eq:tmpd},
across different values of $\beta$ on a suite of benchmark datasets for both
regression and classification tasks. We refer to \eqref{eq:smpd} as SM-PD and
\eqref{eq:tmpd} as TM-PD, where SM and TM stand for standard model and tempered
model, respectively. Notably, our data-driven procedure for selecting $\beta$ is
agnostic to its downstream use in either SM-PD or TM-PD. We demonstrate that the
proposed method performs well in both in terms of test LPD. While our method
aims to maximize test LPD and is supported by the theory discussed in
\cref{sec:supporting-theory}, we also evaluate the point predictions of the PPD.
Specifically, we use $\hat{y} = \E_{p_{\beta}(y | x, \train)}[y]$ to compute
mean squared error (MSE)
$ \frac{1}{|\test|} \sum_{(x, y) \in \test} (y - \hat{y})^{2}$ for regression,
and $\hat{y} = \arg\max_{y} p_{\beta}(y | x, \train)$ to compute accuracy
$\frac{1}{|\test|} \sum_{(x, y) \in \test} \mathds{1}[y = \hat{y}]$ for
classification. By definition, these point predictions $\hat{y}$ are identical
across SM-PD and TM-PD, and their results are consolidated in our reporting. We
compare our method against a grid search over nine
$\beta \in \{ 0.1, 0.3, 1, 3, \ldots, 1000 \}$, corresponding to increments of
roughly 0.5 on the $\log_{10}$ scale. The optimal $\beta$ from the grid search
is selected based on test LPD, MSE, or accuracy on a validation set\footnote{In
  contrast, our method tracks the log-likelihood of $p(y | x, \param, \beta)$ on
  a validation set}, and the optimal $\beta$ may differ depending on the metric
used.

To illustrate our method, we follow the experimental setup of
\citet{wenzel20how}. Regression tasks are conducted using a one-layer ReLU
network on UCI datasets (Concrete, Energy, Naval), while classification is
performed using a CNN on MNIST and a ResNet20 on CIFAR10, both with and without
data augmentation (DA). For the prior on neural network weights, we restrict
$p(\param)$ to a zero-mean isotropic Gaussian, as this is a common choice for
achieving state-of-the-art performance \citep{izmailov21what}. Further details
on the model and prior can be found in \cref{sec:model-prior}.

We employ the stochastic-gradient Markov chain Monte Carlo (SGMCMC) algorithm
from \citep{wenzel20how} with a cyclical step-size scheduler
\citep{zhang20cyclical} to sample from the posterior (details in
\cref{sec:sgmcmc}). Hyperparameters are carefully tuned based on the temperature
diagnostics in \citet{wenzel20how} to ensure sampler convergence to the
posterior with the specified~$\beta$; see \cref{sec:sgmcmc-diagnostics} and
\cref{sec:hyperparameters} for diagnostics and hyperparameters. This approach
contrasts with the common practice of tuning hyperparameters for predictive
performance. For CIFAR10, we collect 30 samples per run, and 100 for all other
datasets. Each SGMCMC run is repeated five times with different initializations
to generate five sets of posterior samples.

\subsection{Results}
\paragraph{Optimal temperature of SM-PD and TM-PD across models and datasets.}
The test LPD across different temperatures for SM-PD and TM-PD are shown in
\cref{fig:test-tempered-lpd-compare}; see \cref{sec:extra-figures} for
additional figures and \cref{tab:lpd-regression} and
\cref{tab:lpd-classification} for tabular versions of the results. For Concrete,
Energy, Naval and MNIST, we observe that TM-PD generally outperforms SM-PD. This
is perhaps not surprising, as we expect the former to better account for
aleatoric uncertainty. For Concrete and MNIST, the peaks of SM-PD and TM-PD also
coincide. Surprisingly, TM-PD does not outperform SM-PD in the CIFAR10 examples.

In terms of the efficacy of our temperature selection method, we find that it
can generally recover $\beta$ with a good test LPD. An interesting observation
here is that our method tends to select the optimal~$\beta$ for TM-PD in the
regression examples, and optimal $\beta$ for SM-PD in the classification
examples. Moreover, as shown in \cref{fig:test-tempered-lpd-compare}, the
optimal $\beta$ for TM-PD and SM-PD for a given model are usually different.
Therefore, we expect $\hat\beta^{*}$, which is computed irrespective of the
construction of the PPD, may only work for either TM-PD or SM-PD, but not both
at the same time.

We also report MSE (for regression) and accuracy (for classification) to assess
the point predictions of the PPDs. As the point predictions are identical across
SM-PD and TM-PD by definition, we consolidate the results in
\cref{tab:test-results}. In general, we observe that SM-PD and TM-PD outperform
both SGD and the PPD at $\beta=1$ across both metrics. Furthermore, our method
can select temperatures that achieves performance comparable to that of the grid
search.

\begin{figure*}[h]
  \vskip 0.2in
  \begin{center}
    \centerline{\includegraphics[width=\linewidth]{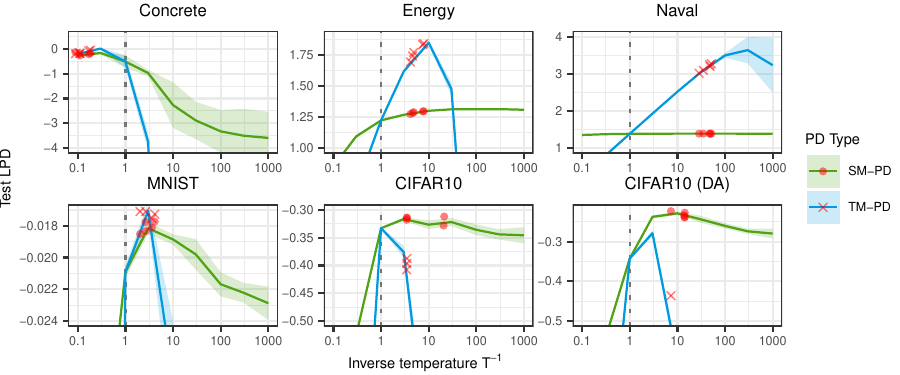}}
    \caption{ Test LPD plotted against inverse temperature $\beta$. We compare
      two types of PPD: SM-PD (green) as defined in \eqref{eq:smpd} and TM-PD
      (blue) as defined in \eqref{eq:tmpd}. Zoomed-in versions of these curves
      are also provided in \cref{fig:test-lpd} and \cref{fig:test-tempered-lpd},
      respectively. In each example, we have five evaluations of
      $\hat \beta^{*}$ from our method. Each of these $\hat \beta^{*}$ has a
      corresponding test LPD computed with SM-PD (red circle) and TM-PD (red
      cross). Some of the red crosses in the CIFAR-10 examples are out of range.
      Solid lines and shaded areas represent the mean $\pm$ standard error
      across five repetitions. The vertical dotted lines indicate the PPD at
      $\beta = 1$. Higher test LPD is better. }
    \label{fig:test-tempered-lpd-compare}
  \end{center}
  \vskip -0.2in
\end{figure*}

\begin{table*}[t]
  \caption{ MSE (for regression) and accuracy (for classification) of the point
    predictions of the PPDs at $\beta = 1$, $\beta = \hat \beta^{*}$ (our
    method), and the optimal $\beta$ obtained from grid search. We also include
    the predictions from SGD as a baseline. The results for SM-PD and TM-PD are
    consolidated, as both produce identical point predictions by definition. The
    presented values are means $\pm$ standard error across five repetitions,
    with the best value among the four methods boldfaced.}
  \label{tab:test-results}
  \vskip 0.15in
  \begin{center}
    \begin{small}
      \begin{sc}
        \begin{tabular}{lcccccc}
          \toprule
          \multirow{2}{*}{Method}  & \multicolumn{3}{c}{MSE $\downarrow$ ($\times 10^{-3}$) } & \multicolumn{3}{c}{Accuracy $\uparrow$}                                                                                                              \\
          \cmidrule(lr){2-4}           \cmidrule(lr){5-7}
                                   & Concrete                                                 & Energy                                  & Naval                      & MNIST                     & CIFAR10                 & CIFAR10 (DA)            \\
          \midrule
          SGD                      & $106 \pm 20$                                             & $1.7 \pm 0.2$                           & $0.149 \pm 0.063$          & $99.05 \pm 0.11$          & $84.9 \pm 3.4$          & $91.9 \pm 0.3$          \\
          $\beta = 1$              & $\mathbf{75 \pm 7}$                                      & $2.1 \pm 0.1$                           & $0.045 \pm 0.005$          & $99.28 \pm 0.04$          & $89.1 \pm 0.3$          & $88.4 \pm 0.2$          \\
          $\beta = \hat \beta^{*}$ & $82 \pm 4$                                               & $1.6 \pm 0.1$                           & $0.032 \pm 0.009$          & $\mathbf{99.38 \pm 0.03}$ & $\mathbf{89.9 \pm 0.2}$ & $\mathbf{92.8 \pm 0.2}$ \\
          Grid                     & $76 \pm 8$                                               & $\mathbf{1.4 \pm 0.1}$                  & $\mathbf{0.027 \pm 0.008}$ & $99.32 \pm 0.05$          & $\mathbf{89.9 \pm 0.2}$ & $\mathbf{92.8 \pm 0.4}$ \\
          \bottomrule
        \end{tabular}
      \end{sc}
    \end{small}
  \end{center}
  \vskip -0.1in
\end{table*}

\paragraph{Optimal temperature across data augmentation strength.}
It has been frequently observed that data augmentation is one of the key ingredients
for observing CPE, and the optimal temperature may depend on the strength of
data augmentation \citep{bachmann22how}. In view of this, we conducted an
experiment to determine how the optimal temperature differs across different
levels of augmentation. Here we only focus on SM-PD and present the
results in \cref{fig:test-lpd-da}. We observe a subtle and gradual increase in
the optimal temperature with the strength of data augmentation. Moreover, our
method can also recover temperatures that produce good test LPD and accuracy
under different augmentation levels. Note that, in addition to the well-known
CPE observed in CIFAR-10 with data augmentation, it also manifests in a milder
form in our CIFAR-10 experiments without data augmentation. This contrasts with
previous studies reporting the absence of CPE in the same setup
\citep{izmailov21what}.

\paragraph{Computation time}
Our approach requires a single SGD run to compute the optimal $\beta$ and a
SGMCMC run to generate the PPD, whereas the grid search requires 9 SGMCMC runs
--- one for each temperature in the grid. The wall-clock times for the main
experiment are reported in \cref{tab:time} of \cref{sec:extra-figures}. Overall,
our method is 4 times faster than the grid search for the smaller regression
models and 8 times faster for the larger classification models.

\begin{figure*}[t]
  \vskip 0.2in
  \begin{center}
    \centerline{\includegraphics[width=\linewidth]{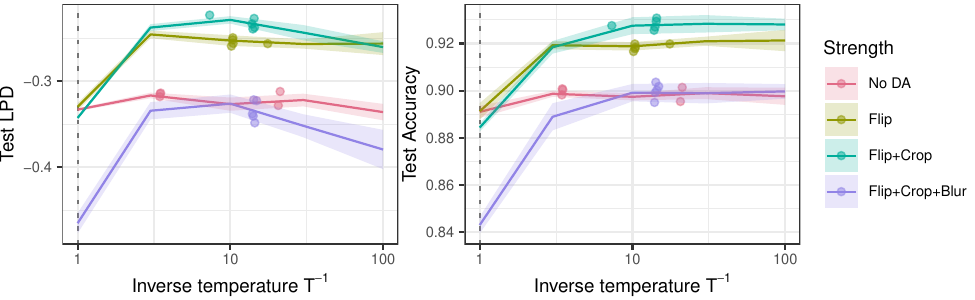}}
    \caption{ Test LPD and accuracy of CIFAR-10 plotted against inverse
      temperature $\beta$ under various levels of data augmentation (color). The
      lines and shaded areas represent the mean $\pm$ standard error across five
      repetitions. There are five dots for each colored curve, and each of these
      dots corresponds to a repetition of $\hat \beta^{*}$ from our method. The
      vertical dotted lines indicate the PPD at $\beta = 1$. There is a subtle
      shift of peaks from left to right as the augmentation strength increases.
      Higher test LPD and accuracy indicate better performance. }
    \label{fig:test-lpd-da}
  \end{center}
  \vskip -0.2in
\end{figure*}

\section{Related Work: CPE through the Lens of Bayesian Deep Learning}
\label{sec:bdl-lens}
Since the publication of \citet{wenzel20how}, there have been numerous attempts
to explain CPE, with particular focus on predictive metrics such as test LPD,
accuracy (for classification), and MSE (for regression). In addition to the
arguments presented in \cref{sec:how-pick-right}, we provide a summary of the
popular insights into CPE from the BDL literature here.

\paragraph{Poor posterior approximation.}
One explanation for the presence or absence of CPE is inadequate posterior
approximation, e.g.,~an inappropriate choice of step size, or the omission of
Metropolis-Hastings steps in the SGMCMC sampler. However, as shown in
\cref{sec:div-truth-ppd,sec:optimality-test-lpd}, CPE is not solely an artifact
of poor posterior approximation and can be observed theoretically.

\paragraph{Likelihood corruption due to data augmentation.}
It has been widely observed that the use of data augmentation often amplifies
CPE \citep{izmailov21what}, while being less pronounced when data augmentation
is turned off. It is argued that the augmented data violate the usual
independent and identically distributed (i.i.d.)~assumption imposed on the
data. However, CPE has been found to persist even after accounting for this
assumption violation \citep{nabarro22data}. Therefore, CPE is unlikely to be a
mere artifact of data augmentation. In a separate analysis,
\citet{kapoor22uncertainty} argues that the SGMCMC sampler will converge to a
tempered posterior in the presence of data augmentation. They concluded that the
likelihood is implicitly raised to a power equal to the number of augmentations,
and raising the likelihood to a power reciprocal of this number should
approximately recover the standard posterior. However, they do not find this
adjustment alone sufficient to remove CPE completely.

\paragraph{Model misspecification.}
It has been argued that data augmentation and curation may lead to model
misspecification, in particular overestimating the aleatoric uncertainty in the
data \citep{aitchison21statistical,kapoor22uncertainty}. Therefore, tempering is
proposed to be an effective tool for correction. This finding aligns with
\citet{bachmann22how}, which showed the optimal temperature dependent on both
the aleatoric uncertainty in the data and the `invariance' of the model to
augmented data.

\paragraph{Prior misspecification.}
While the exact interpretation of a `well-specified prior' is debatable,
\citet{fortuin22bayesian} conducted a large-scale experiment to study the effect
of tempering under Gaussian (isotropic or with correlated covariance) and
heavy-tailed priors (Student's-t or Laplace) on four metrics: test LPD,
accuracy, expected calibration error \citep{naeini15obtaining}, and
out-of-distribution detection accuracy. They observed that $\beta = 1$ is indeed
optimal for test LPD and accuracy when using a heavy-tailed prior. However, the
standard posteriors derived from heavy-tailed priors also tend to underperform
compared to tempered posteriors derived from Gaussian priors. Notably, warmer
temperatures tend to lower expected calibration error, and there is no general
trend in optimal $\beta$ for out-of-distribution detection accuracy. Therefore,
their work suggests that the optimal $\beta$ is not only dependent on the
truth-model-prior triplet but also on the evaluation metric.

\paragraph{Attempts to `fix' the CPE.} Within much of the CPE literature,
tempering is often seen as a `hack' that strays from the Bayesian principle.
This has prompted the development of various `fixes' --- models and priors that
induce standard posteriors with similar predictive performance to tempered
posteriors \citep[see,
  e.g.,][]{fortuin22bayesian,kapoor22uncertainty,marek24can}. However, we argue
that this is unnecessary, as tempered posteriors of arbitrary $\beta > 0$ are
special cases of the generalized Bayes posterior \citep{bissiri16general}; see
\cref{sec:bissiri-primer} for a primer. Therefore, posterior tempering does not
deviate from the Bayesian principle.

\section{Related Work: CPE through the Lens of Generalized Bayes}
\label{sec:model-free-bayesian}
In GB, it is common to tune $\beta$ to improve the utility of the tempered
posterior \citep{zhang06epsilon,grunwald12safe,bissiri16general}. Specifically,
the likelihood is frequently (but not always) tempered to a warmer temperature,
which contrasts with the cold temperature used in BDL. In this section, we
review some recent developments in GB to clarify the role of tempering in GB and
this apparent contradiction.

\paragraph{Ensuring posterior concentration on the KL minimizer.} Posterior
concentration on the KL minimizer
$\param_\dag = \argmin_{\param} \KL(q(y|x) | p(y|x,\param))$\footnote{In GB, the
  minimizer is assumed to be unique, though this is often not true for neural
  network models.} as $\ntrain \to \infty$ is often a desired property for many statistical
applications. Proofs establishing this property
\citep{barron91minimum,zhang06epsilon,grunwald07minimum} typically assume that
the following inequality holds for all $\param$ in the parameter space:
\begin{equation*}
  \E_{q(x,y)} \left[ \left(\frac{p(y | x, \param)}{p(y | x, \param_\dag)} \right)^{\beta} \right] \leq 1,
  \quad \text{for all } \param \in \setparam.
\end{equation*}
This inequality holds at $\beta = 1$ for well-specified models, where
\( q(x, y) = p(y | x, \param_\dag) q(x) \), but the same cannot be said for
misspecified models in general. \citet{grunwald12safe} shows that this
inequality can still hold for many misspecified models for some warm
temperatures $\beta \leq \beta_{\text{critical}} < 1$, and proposed the
SafeBayes algorithm to determine $\beta_{\text{critical}}$. Therefore, SafeBayes
is presented as a tool to achieve posterior concentration on $\param_\dag$.

However, posterior concentration is not always desirable in the context of
maximizing test LPD under model misspecification. By taking
$\beta \leq \beta_{\text{critical}}$, the corresponding PPD will concentrate on
the `plug-in' predictive density $p(y | x, \param_\dag)$ as
$\ntrain \to \infty$. This may be undesirable, since $p(y | x, \param_\dag)$ may
not exhibit good test LPD. For example, when the misspecified model is
non-convex, the PPD can lie outside the model family and be closer to the truth
than $p(y | x, \param_{\dag})$. Therefore, avoiding posterior concentration in
this situation can actually result in a PPD with a higher test LPD.

\paragraph{Calibrating credible and prediction intervals.}
Constructing well-calibrated credible and prediction intervals with the nominal
frequentist coverage probability is known to be challenging, as obtaining the
posterior variance of a misspecified model is difficult \citep[Example
  2.1]{kleijn12bernsteinvonmises}. These issues can be mitigated by appropriately
tuning $\beta$, which affects the spread of the posterior. To this end,
\citet{syring19calibrating} and \citet{wu21calibrating} developed algorithms to
select $\beta$ for calibrating credible and prediction intervals, respectively.

\paragraph{Calibrating prior-to-posterior information gain.}
In decision-theoretic GB \citep[see~\cref{sec:bissiri-primer} for a
  primer]{bissiri16general}, tempered posteriors (with likelihood tempering only)
are seen as an update rule that combines data information with prior belief.
From this perspective, $\beta$ can be interpreted as a `learning rate' at which
information is `transferred' to the posterior. Hence, it is reasonable to
calibrate the information gain at each update. Temperature selection algorithms
that follow this approach include \citet{holmes17assigning} and
\citet{lyddon19general}.

\paragraph{Improving test LPD.}
Concurrent to our work, \citet{mclatchie24predictive} analyzed the role of
$\beta$\footnote{In \citet{mclatchie24predictive}, the \textit{learning rate}
  and, confusingly, \textit{temperature} refer to $\beta$ in this work.} in
improving test LPD theoretically. They concluded that, in the moderately large
$\ntrain$ regime and assuming posterior concentration at $\param_{\dag}$, the
test LPD shows little improvement once $\beta$ becomes sufficiently large.
Although their motivation aligns with ours, their theoretical results do not
extend to neural networks or over-parameterized models (e.g., ResNet20 on
CIFAR10).

\subsection{Limitations of Generalized Bayes for Analyzing BDL Models}
The existing theory in GB, though elegant, is limited to regular and
under-parameterized models and cannot fully account for CPE observed in BDL.
Moreover, the metrics emphasized by GB often differ from those prioritized by
BDL practitioners, such as test LPD. As a result, GB works do not provide a
prescriptive methodology for selecting an appropriate temperature in BDL.
Additionally, many current temperature selection algorithms
\citep{grunwald12safe,syring19calibrating,wu21calibrating} require repeated
posterior computations, making them impractical for modern deep learning models.
Finally, the GB literature primarily considers posteriors with likelihood
tempering only, i.e.,~leaving the prior without tempering. However, since many
priors $p(\param)$ encountered in BDL are proper and bounded, prior tempering
can be regarded as rescaling the prior \citep[Section C.2]{kapoor22uncertainty}.
Therefore, many of the arguments in GB still apply.

\section{Conclusion}
\label{sec:conclusion}
In this work, we proposed a data-driven approach for selecting a good $\beta$
for use in either of the PPDs \eqref{eq:smpd} and \eqref{eq:tmpd}. Our approach
circumvents the costly grid search method, i.e.,~sampling posterior at each
$\beta$ across the grid, by optimizing a likelihood function \eqref{eq:temp-mle}
to obtain a good $\beta$. The $\hat \beta^{*}$ obtained via our method was shown
to achieve comparable test LPD to that obtained from a grid search, all without
performing any extra posterior sampling.

Additionally, we presented a detailed discussion to address the seeming
contradiction in the optimal~$\beta$ recommendations from the BDL and GB
communities. We concluded that the optimal $\beta$ can differ depending on the
specific downstream task.

\paragraph{Limitation.} Our method is subject to several limitations that require
future attention. Firstly, a poorly-tuned optimization procedure may result in a
poor estimation of $\hat \beta^{*}$. Secondly, our method has only been
empirically tested with the ubiquitous Gaussian priors. Lastly, while our method
is empirically effective for modern neural network models and theoretically
justified for Gaussian linear regression, a formal guarantee for general neural
network models is still lacking.

\paragraph{Future directions.}
We hope that this work will inspire further research on data-driven approaches
to select $\beta$, aimed not only at minimizing test~LPD but also at optimizing
other metrics, such as expected calibration error.

\section{Reproducibility Statement}
The experiments details are summarized in the introduction of
\cref{sec:experiments}. Proof of \cref{thm:linear-lpd} is given
in~\cref{sec:proof-linear-lpd}. SGMCMC is described in \cref{sec:sgmcmc}.
Diagnostic check of the SGMCMC is given in \cref{sec:sgmcmc-diagnostics}.
Implementation details of our proposed method is given in
\cref{sec:algorithm-structure}. The set of hyperparameters to reproduce the
experiments is given in \cref{sec:hyperparameters}. Computational environment
and packages are given in \cref{sec:comp-envir}. Evaluation metrics are defined
in~\cref{sec:evaluation-metrics}. Model, prior and datasets are described in
\cref{sec:model-prior}. Source code is published on
\url{https://github.com/weiyaw/tempered-posteriors}.

\section{Acknowledgment}
We would like to thank David Frazier, Dino Sejdinovic and the anonymous
reviewers for their helpful discussions. This research was supported by Cloud
TPUs from Google’s TPU Research Cloud. K.N. is supported by the Australian
Government’s Research Training Program and a joint top-up scholarship from the
Statistical Society of Australia and the Australian Bureau of Statistics. L.H.
is supported by the Australian Research Council through a Discovery Early Career
Researcher Award (DE240100144).

\bibliography{main}
\bibliographystyle{tmlr}

\appendix
\section{Optimal Temperature of a Gaussian Toy Model}
\label{sec:div-truth-ppd}
In this section, we compute the optimal temperature that minimizes the
2-Wasserstein distance and the KL divergence between the truth and PPD of a toy
Gaussian model. We show that: 1) the optimal temperature depends on both the
evaluation metric and the truth-model-prior triplet, and 2) the optimal
temperature that maximizes test LPD is not necessarily~1, even in the case of a
well-specified model.

In our setup, we assume a set of i.i.d.~samples
$\train = \{x_{i}\}_{i = 1}^{\ntrain}$ drawn from a univariate Gaussian truth
$q(x) = \mathcal{N}(x | 0, \tau^2)$ with a known variance $\tau^{2}$. Our model
is given by $p(x | \mu) = \mathcal{N}(x | \mu, \sigma^2)$ with a fixed variance
$\sigma^{2}$. We also assume a Gaussian prior on $\mu$,
i.e.,~$p(\mu) = \mathcal{N}(\mu | 0, \sigma^2_p)$. While we focus on the
unsupervised setting for the ease of exposition, our argument also applies to
the supervised settings.

The tempered posterior at $\beta = \frac{1}{\T}$ is given by:
\begin{align*}
  p_\beta(\mu | \train)
   & \propto p(\mu)^{1/T} \prod_{i=1}^{n} \mathcal{N}(x_i | \mu, \sigma^2 T) \\
   & = \mathcal{N}\left(
  \mu \middle|
  n \bar x \frac{\sigma^2_{post}}{\sigma^2 T},
  \sigma^{2}_{post}\right)                                                   \\
   & = \mathcal{N}\left(
  \mu \middle|
  \bar x \left( \frac{\sigma^2}{n\sigma^2_p} + 1 \right)^{-1},
  \sigma^{2}_{post}\right),
\end{align*}
where the posterior variance is given by
$\sigma^{2}_{post} = T \left(\frac{1}{\sigma^{2}_p} + \frac{n}{\sigma^{2}} \right)^{-1}$.

The PPD as defined in~\eqref{eq:smpd} is then:
\begin{equation*}
  p_\beta(x | \train)
  = \int \mathcal{N}(x | \mu, \sigma^2) p_{\beta}(\mu | \train) \textrm{d} \mu
  = \mathcal{N}\left(
  x \middle|
  n \bar x \frac{\sigma^2_{post}}{\sigma^2 T},
  \sigma^2_{post} + \sigma^2
  \right)
\end{equation*}

\subsection{Optimal Temperature for Minimizing 2-Wasserstein Distance}
As both $q(x)$ and $p_{\beta}(x | \train)$ are Gaussian, the 2-Wasserstein
distance $W_{2}$ between these two density has a closed-form expression:
\begin{equation*}
  W_2 = \left(\left(\frac{\sigma^2}{n \sigma^2_p} + 1\right)^{-1} \bar x \right)^2
  + T \left(\frac{1}{\sigma^{2}_p} + \frac{n}{\sigma^{2}} \right)^{-1}
  + \sigma^2 + \tau^2
  - 2\tau \left[ T \left(\frac{1}{\sigma^{2}_p} + \frac{n}{\sigma^{2}} \right)^{-1} + \sigma^2 \right]^{1/2}
\end{equation*}

Taking the gradient of $W_{2}$ with respect to $\T$ and set it to 0, we can
derive the optimal temperature~$T^{*}$:
\begin{gather*}
  \pdv{W_2}{T}
  = \left(\frac{1}{\sigma^{2}_p} + \frac{n}{\sigma^{2}} \right)^{-1}
  - \tau \left[ T \left(\frac{1}{\sigma^{2}_p} + \frac{n}{\sigma^{2}} \right)^{-1} + \sigma^2 \right]^{-1/2} \left(\frac{1}{\sigma^{2}_p} + \frac{n}{\sigma^{2}} \right)^{-1}
  = 0 \\
  \implies  T^* = (\tau^2 - \sigma^2) \left(\frac{1}{\sigma^{2}_p} + \frac{n}{\sigma^{2}} \right), \quad T^{*} \in (0, \infty)
\end{gather*}
Therefore, we can conclude that the optimal $T^{*} \to 0$ in the well-specified
case ($\tau^{2} = \sigma^{2}$).

\subsection{Optimal Temperature for Minimizing KL Divergence}
We can also obtain the closed-form expression KL divergence,
$\KL(q(x) \| p_{\beta}(x | \train))$:
\begin{align*}
  \KL(q(x) \| p_{\beta}(x | \train))
   & = \E_{q( x)} \left[ \log \frac{q(x)}{p_{\beta}( x | \train)} \right]                                                           \\
   & = \left[ \left(\frac{\sigma^2}{n \sigma^2_p} + 1 \right)^{-1} \bar x  \right]^2 \frac{1}{2 (\sigma^2_{post} + \sigma^2)}       \\
   & \quad + \frac{1}{2} \left( \frac{\tau^2}{\sigma^2_{post} + \sigma^2} -1 -\ln \frac{\tau^2}{\sigma^2_{post} + \sigma^2}\right).\end{align*}
Note that this object directly corresponds to test log predictive density (LPD),
which is the primary object of interest in Bayesian deep learning and our work.
Taking the derivative of the divergence with respect to $T$ and set it to 0, we
can solve for the posterior variance that minimizes the divergence:
\begin{gather*}
  \pdv{\KL}{T}
  = \left[ -\left[ \left(\frac{\sigma^2}{n \sigma^2_p} + 1 \right)^{-1} \bar x  \right]^2 \frac{1}{ (\sigma^2_{post} + \sigma^2)^2}
  + \left( \frac{-\tau^2}{(\sigma^2_{post} + \sigma^2)^2} +\frac{1}{\sigma^2_{post} + \sigma^2}\right) \right] \frac{1}{2} \pdv{\sigma_{post}^2}{T}
  = 0 \\
  \implies \sigma^2_{post} = \left[ \left(\frac{\sigma^2}{n \sigma^2_p} + 1 \right)^{-1} \bar x \right]^2 + \tau^2 - \sigma^2
\end{gather*}
From this posterior variance, we can deduce the optimal $T^*$:
\begin{align*}
  T^*
   & = \left[
    \left[ \left( \frac{\sigma^2}{n \sigma^2_p} + 1  \right)^{-1} \bar x \right]^2 + \tau^2 - \sigma^2
    \right]
  \left( \frac{1}{\sigma^{2}_p} + \frac{n}{\sigma^{2}} \right) \\
   & = \left( \frac{n}{\sigma^2} \bar x \right)^2
  \left( \frac{1}{\sigma^{2}_p} + \frac{n}{\sigma^{2}} \right)^{-1}
  + (\tau^2 - \sigma^2) \left( \frac{1}{\sigma^{2}_p} + \frac{n}{\sigma^{2}} \right), \quad \T^{*} \in (0, \infty).
\end{align*}
We can see that $T^{*}$ is non-trivial, and $T^{*} \neq 1$ in the well-specified
case ($\tau^{2} = \sigma^{2}$). This contradicts with the popular belief that
the test LPD is maximized by the PPD induced from a well-specified standard
posterior \citep{adlam20cold}.

\section{Optimality Conditions for Test LPD at $\beta=1$}
\label{sec:optimality-test-lpd}
In this section, we follow the work of \citet{zhang24cold} (which assumes
likelihood tempering rather than full posterior tempering) and argue that the
PPD \eqref{eq:smpd} is rarely optimal in terms of test LPD at $\beta = 1$. Our
approach is to demonstrate that the gradient
$\nabla_\beta \LPD(p_{\beta}(y | x, \train))$ is rarely zero at $\beta = 1$.
\begin{lemma}
  \label{thm:grad-lpd}
  Let
  $p(\train, \param) = p(\param) \prod_{(x, y) \in \train} p(y | x, \param)$,
  where we suppress the conditional dependency on $x$ in $p(\train, \param)$.
  Then, the gradient of $\LPD(p_{\beta}(y | x, \train))$ with respect to $\beta$
  is given by
  \begin{equation*}
    \nabla_\beta \LPD(p_{\beta}(y | x, \train))
    = \E_{q(x, y)} \left[\E_{p_{\beta}(\param | \train \cup (x, y))} [\log p(\train , \param)]\right]
    - \E_{p_{\beta}(\param | \train)} [\log p(\train, \param)],
  \end{equation*}
  where
  $p_{\beta}(\param | \train \cup (x, y)) \propto p(y | x, \param) p_{\beta}(\param | \train)$
  represents an update to the posterior $p_{\beta}(\param | \train)$ after
  observing an extra data point $(x, y)$ from the truth, and the expectation
  $\E_{q(x, y)}[\cdot]$ integrates over $(x, y)$ in
  $p(\param | \train \cup (x, y))$.
\end{lemma}

\begin{proof}
  Let $\E_{p_\beta}[\cdot] \coloneq \E_{p_{\beta}(\param | \train)}[\cdot]$
  represent the expectation with respect to the tempered posterior
  $p_{\beta}(\param | \train)$, and let the normalizing constant of the tempered
  posterior be
  $Z(\train, \beta) \coloneq \int p(\train, \param)^{\beta} \mathrm{d}{\param}$.

  We begin by deriving an identity that will be useful for the gradient computation. For any arbitrary $f: \R^{\dimparam} \to \R$, the following identity holds:
  \begin{align*}
    \nabla_\beta \E_{p_{\beta}} f(\param)
     & = \int \nabla_\beta [f(\param) p_{\beta}(\param | \train)] \mathrm{d}\param                               \\
     & = \int f(\param) \nabla_\beta \log p_{\beta}(\param | \train) p_{\beta}(\param | \train) \mathrm{d}\param \\
     & = \E_{p_{\beta}}[f(\param) \nabla_\beta \log p_{\beta}(\param | \train)]                                  \\
     & = \E_{p_{\beta}}[f(\param) \nabla_\beta \log p(\train , \param)^{\beta}] -
    \E_{p_{\beta}}[f(\param)] \nabla_\beta \log Z(\train, \beta)                                                 \\
     & = \E_{p_{\beta}}[f(\param) \log p(\train, \param)] -
    \E_{p_{\beta}}[f(\param)] \E_{p_{\beta}} [\log p(\train, \param)],
  \end{align*}
  where $\nabla_\beta \log Z(\train, \beta)$ in the second last line becomes
  \begin{align*}
    \nabla_\beta \log Z(\train, \beta)
    = \frac{\int \nabla_{\beta} p(\train, \param)^{\beta} \mathrm{d}\param}{Z(\train, \beta)}
    = \int \log p(\train, \param) \frac{p(\train, \param)^{\beta}}{Z(\train, \beta)} \mathrm{d}\param
    = \E_{p_{\beta}} [\log p(\train, \param)].
  \end{align*}

  The gradient can then be computed as follows:
  \begin{align*}
    \nabla_{\beta} \LPD(p_{\beta}(y | x, \train))
     & = \nabla_{\beta} \E_{q(x, y)} \log \E_{p_{\beta}} p(y | x, \param)                                   \\
     & = \E_{q(x, y)} \frac{\nabla_{\beta} \E_{p_{\beta}}p(y | x, \param)}{\E_{p_{\beta}} p(y | x, \param)} \\
     & = \E_{q(x, y)} \left[
      \frac{\E_{p_{\beta}}[p(y | x, \param) \log p(\train, \param)]}{\E_{p_{\beta}} p(y | x, \param)}
      \right]
    - \E_{p_{\beta}} [\log p(\train, \param)]                                                               \\
     & = \E_{q(x, y)} [\E_{p_{\beta}(\param | \train \cup (x, y))} \log p(\train, \param)]
    - \E_{p_{\beta}} [\log p(\train, \param)],
  \end{align*}
  where the term inside $\E_{q(x, y)}[\cdot]$ in the third equality simplifies to
  \begin{align*}
    \frac{\E_{p_{\beta}}[p(y | x, \param) \log p(\train, \param)]}{\E_{p_{\beta}} p(y | x, \param)}
     & = \int  \log p(\train, \param)
    \frac{p(y | x, \param) p_{\beta}(\param | \train)}{\int p(y | x, \param') p_{\beta}(\param' | \train) \mathrm{d} \param'} \mathrm{d}\param \\
     & = \E_{p_{\beta}(\param | \train \cup (x, y))} \log p(\train, \param).
  \end{align*}
\end{proof}
We can view the terms in the RHS of \cref{thm:grad-lpd} as a training loss under
a posterior $\rho$:
\begin{equation*}
  \E_{\rho} L(\train, \param)
  \coloneq \E_{\rho} [- \log p(\train, \param)]
  = \E_{\rho}[- \log p(\param) - \sum_{(x, y) \in \train} \log p(y | x, \param)].
\end{equation*}
This is similar to the training loss as defined in \citet{zhang24cold}, but with
an additional regularizer $p(\param)$ included in the loss. Then, we can define
\textit{underfitting} as
$\E_{q(x, y)} [\E_{p_{\beta}(\param | \train \cup (x, y))} L(\train, \param)] < \E_{p_{\beta}(\param | \train)} L(\train, \param)$,
i.e.,~the posterior $p_{\beta}(\theta | \train \cup (x, y))$ will, on average,
have a lower training loss than the original $p_{\beta}(\theta | \train)$ after
receiving an extra observation $(x, y)$ from the truth. Similarly, we can define
the converse as \textit{overfitting}.

Therefore, it is not difficult to imagine that the posterior is rarely
well-fitted at $\beta = 1$, and we should not expect optimality in terms of the
test LPD from the PPD at $\beta = 1$.

\section{Details of the Temperature Selection Algorithm}
\label{sec:algorithm-structure}
Our proposed procedure for constructing a PPD at an optimal temperature is
summarized in \cref{alg:tempered-posterior} with the following details:
\begin{enumerate}
  \item We first run SGD to maximize the log-likelihood of the tempered model on
        a training set. We track the validation log-likelihood at the end of
        every $\lfloor L / 20 \rfloor$-th epoch, where $L$ is the total number
        of epochs, to save computation time. We then select the temperature with
        the largest validation log-likelihood as our optimal temperature $\hat \beta^{*}$. \\
  \item Subsequently, we run SGMCMC to draw samples from the tempered posterior
        at $\hat \beta^{*}$. We only start collecting samples after a burn-in
        phase. During the burn-in phase, we use a linear ramp function to
        control $\T = \frac{1}{\beta}$ in \eqref{eq:sgmcmc-update-m}, i.e.~we
        first run SGMCMC at $\T=0$, then ramp up $\T$ from 0 to
        $\hat \T^{*} = \frac{1}{\hat \beta^{*}}$ and continue the burn-in at
        $\hat \T^{*}$. This SGMCMC scheme is almost identical to \citet[Appendix
          A.1]{wenzel20how}, except that we run extra burn-in epochs at $\hat \T^{*}$. \\
\end{enumerate}

\begin{algorithm}
  \caption{Procedure to construct a PPD at the optimal temperature.}
  \label{alg:tempered-posterior}
  \KwIn{Training data $\train$, Validation data $\valid$}

  Initialize $\param \gets \text{some initializer}, \beta \gets 1$\;

  Set $(\hat \param^{*} , \hat \beta^{*}) \gets (\param,  \beta )$\;

  \Repeat{Convergence or resource exhausted}{
    $\Delta \gets$ Compute and scale $\nabla_{\param, \log \beta}\sum_{(x, y) \in \train} \log p(y | x, \param, \beta)$\;

    $(\param, \log \beta) \gets (\param, \log \beta) + \Delta$\;

    \If{
      $\sum_{(x, y) \in \valid} \log p(y | x, \param, \beta) >
        \sum_{(x, y) \in \valid} \log p(y | x, \hat \param^{*}, \hat \beta^{*})$
    }{
      $(\hat \param^{*}, \hat \beta^{*}) \gets (\param, \beta)$
    }
  }

  $\mathcal{S} \gets$ Draw $\param$ from the tempered posterior
  $p_{\beta}(\param | \train)$ with SGMCMC at $T = \frac{1}{\hat \beta^{*}}$\;

  \KwOut{ The PPD
  $p_{\beta}(y | x, \train) \approx |\mathcal{S}|^{-1} \sum_{\param \in \mathcal{S}} p(y | x, \param)$
  from the set of samples $\mathcal{S}$
  }
\end{algorithm}

\section{Variants of the Temperature Selection Method}
\label{sec:alternative-strategies}
In this section, we compare a variant of \eqref{eq:temp-mle} by solving
\begin{equation}
  \argmax_{\param, \beta} \frac{1}{\ntrain} \sum_{(x,y) \in \train}[\log p(y | x, \param, \beta ) + \log p(\param)].
\end{equation}
We refer to this approach as the \textit{maximum-a-posteriori} (MAP) method, as
$\param$ is now constrained by a prior $p(\param)$. This contrasts with the
\textit{maximum likelihood} (MLE) method introduced in~\eqref{eq:temp-mle}.
Additionally, we compare our method to the popular temperature scaling approach
proposed in \citet{guo17calibration}. In their method, $\param$ in the
objective~\eqref{eq:temp-mle} is fixed to a solution obtained from the standard
training workflow (e.g., $\param^{*}_{\text{SGD}}$), and the temperature $\beta$
is optimized by solving
\begin{equation*}
  \argmax_{\beta} \frac{1}{\ntrain} \sum_{(x,y) \in \valid} \log p(y | x, \param^{*}_{\text{SGD}}, \beta)
\end{equation*}
on a validation set. In contrast, both MLE and MAP jointly optimizes $\param$ and $\beta$ within a single SGD run.

We compare the methods on the CIFAR-10 experiments, reporting the
test LPD and accuracy in \cref{tab:strategy}. The PPDs are constructed from
SM-PD~\eqref{eq:smpd}. In general, MAP and MLE yield similar $\beta$ values,
while the $\beta$ obtained from \citet{guo17calibration} differs
substantially. Moreover, \citet{guo17calibration} generally underperforms
relative to MAP and MLE when the CPE is most pronounced, such as when data
augmentation is enabled. This outcome is perhaps unsurprising, as
\citet{guo17calibration} is designed for computing a well-calibrated tempered
model $p(y | x, \param, \beta)$, whereas our method targets the PPD
$p_{\beta}(y | x, \train)$.

The results also suggest that optimizing $\param$ and $\beta$ together is
critical for determining a good $\beta$. We conjecture that
$p(y | x, \param^{*}, \beta^{*})$, with both $\param^{*}$ and $\beta^{*}$
obtained through our procedure, provides a reasonable approximation of the PPD
with the highest test LPD. In contrast, fixing $\param$ to a predetermined value
(e.g., $\param^{*}_{\text{SGD}}$) limits the search space and leads to a poor
approximation of the PPD.

Another possible strategy here is the `marginal likelihood method' which selects
$\beta$ that maximizes the marginal likelihood of the tempered posterior.
However, marginal likelihood is generally computationally intractable. Existing
methods that rely on Laplace approximations \cite{immer21scalable} are generally
unreliable due to the singularity of neural network models \cite{wei22deep},
although progress has been made to address this limitation
\citep{hodgkinson23interpolating}.

\begin{table}
  \caption{Comparison between methods to select $\beta$ on the CIFAR10
    experiments. Reported values are mean $\pm$ standard error across five
    repetitions. All values are evaluated on a test set. The accuracy and LPD
    are all within one standard error of difference between MAP and MLE.}
  \label{tab:strategy}
  \centering
  \begin{tabular}{lcccc}
    \toprule
    Data                          & Method                   & Accuracy $\uparrow$ & LPD $\uparrow$     & $\hat \beta^{*}$ \\
    \midrule
    \multirow{3}{*}{CIFAR10}      & \citet{guo17calibration} & $89.86 \pm 0.2$     & $-0.317 \pm 0.003$ & $1.75 \pm 0.18$  \\
                                  & MAP                      & $89.89 \pm 0.25$    & $-0.321 \pm 0.010$ & $10.54 \pm 9.60$ \\
                                  & MLE                      & $89.92 \pm 0.25$    & $-0.317 \pm 0.006$ & $10.44 \pm 9.54$ \\
    \midrule
    \multirow{3}{*}{CIFAR10 (DA)} & \citet{guo17calibration} & $90.47 \pm 0.14$    & $-0.286 \pm 0.004$ & $1.61 \pm 0.05$  \\
                                  & MAP                      & $92.81 \pm 0.27$    & $-0.232 \pm 0.008$ & $14.28 \pm 0.52$ \\
                                  & MLE                      & $92.81 \pm 0.20$    & $-0.232 \pm 0.007$ & $12.79 \pm 3.06$ \\
    \bottomrule
  \end{tabular}
\end{table}

\section{Proof of Lemma~\ref{thm:linear-lpd}}
\label{sec:proof-linear-lpd}

\begin{lemma}
  Consider a linear regression model
  $p(y | x, \param) = \mathcal{N}(y | x^{\top} \param, \sigma^{2})$ with a
  $\dimparam$-dimensional input $x$ and known variance $\sigma^{2}$, and a prior
  $p(\param) = \mathcal{N}(\param | 0, \sigma^{2}_{p})$ with finite variance
  $\sigma^{2}_{p}$. Let
  $\bm{X} \coloneq (x_{1}, \ldots, x_{\ntrain})^{\top} \in \R^{\ntrain \times \dimparam}$
  and
  $\bm{\Sigma} \coloneq (\bm{X}^{\top}\bm{X} + \frac{\sigma^{2}}{\sigma^{2}_{p}} \bm{I})^{-1}$.
  The test LPD of the PPD in~\eqref{eq:tmpd} at a fixed $\beta$ is bounded from
  below:
  \begin{equation*}
    \LPD(\E_{p_{\beta}(\param | \train)}[p(y | x, \param, \beta)])
    > \E_{q(x, y)}  \log p(y | x, \hat{\param}_{\textrm{MAP}}, \beta)
    -\frac{1}{2} \E_{q(x, y)}  \log{(1 + x^{\top} \bm{\Sigma} x)} ,
  \end{equation*}
  where $\hat{\param}_{\textrm{MAP}} \coloneq \bm{\Sigma} \bm{X}^{\top} \bm{y}$
  is the \emph{maximum-a-posteriori} solution of the posterior
  $p_{\beta}(\param | \train)$ at $\beta = 1$ and
  $\bm{y} \coloneq (y_{1}, \ldots y_{\ntrain})^{\top} \in \mathbb{R}^{\ntrain}$.
\end{lemma}

\begin{proof}
  Let $\sigma^{2}_{\beta} = \frac{\sigma^{2}}{\beta}$. The tempered posterior
  can be shown to follow a Gaussian distribution
  \begin{equation*}
    p_{\beta}(\param | \train)
    \propto \mathcal{N}(\param | 0, \sigma^{2}_{p} / \beta)
    \prod_{x, y \in \train} \mathcal{N}(y | x^{\top} \param, \sigma^{2}_{\beta})
    \propto \mathcal{N}(\param | \hat{\param}, \sigma^{2}_{\beta} \bm{\Sigma}) .
  \end{equation*}
  Therefore, the PPD~\eqref{eq:tmpd} of linear regression can be derived using
  the identities of conditional Gaussian densities \citep[see][Section
    2.3.3]{bishop06pattern}
  \begin{align*}
    \log \E_{p_{\beta}(\param | \train)}[p(y | x, \param, \beta)]
     & = \log \int p(y | x, \param, \beta) p_{\beta}(\param | \train) \textrm{d}\param                \\
     & = \log \int \mathcal{N}(y | x^{\top} \param, \sigma^{2}_{\beta})
    \mathcal{N}(\param | \hat{\param}, \sigma^{2}_{\beta} \bm{\Sigma}) \textrm{d}\param               \\
     & = \log \mathcal{N}(y | x^{\top} \hat{\param}, \sigma^{2}_{\beta} (1 + x^{\top} \bm{\Sigma} x)) \\
     & = - \frac{1}{2} \log(1 + x^{\top} \bm{\Sigma} x)
    - \frac{1}{2} \log(2 \pi \sigma^{2}_{\beta})
    - \frac{(y - x^{\top} \hat{\param})^{2}}{2 \sigma^{2}_{\beta} (1 + x^{\top} \bm{\Sigma} x)}       \\
     & > - \frac{1}{2} \log(1 + x^{\top} \bm{\Sigma} x)
    - \frac{1}{2} \log(2 \pi \sigma^{2}_{\beta})
    - \frac{(y - x^{\top} \hat{\param})^{2}}{2 \sigma^{2}_{\beta} }                                   \\
     & = -\frac{1}{2} \log{(1 + x^{\top} \bm{\Sigma} x)}
    + \log \mathcal{N}(y | x^{\top} \hat{\param}, \sigma^{2}_{\beta})                                 \\
     & = -\frac{1}{2} \log{(1 + x^{\top} \bm{\Sigma} x)}
    + \log p(y | x, \hat{\param}, \beta)
  \end{align*}
  Note that $\bm{\Sigma}$ is positive definite, and thus
  $1 + x^{\top} \bm{\Sigma} x > 1$ and the inequality holds. Taking expectation
  with respect to $q(x, y)$ at both sides of the inequality concludes the claim.
\end{proof}

\section{Evaluation Metrics}
\label{sec:evaluation-metrics}
In addition to the LPD evaluated on a test set $\test$,
\begin{equation*}
  \E_{q(x, y)} [\log p_{\beta}(y | x, \train)]
  \approx \dfrac{1}{|\test|} \sum_{x, y \in \test} \log p_{\beta}(y | x, \train),
\end{equation*}
we also report MSE for the regression examples
\begin{equation*}
  \dfrac{1}{|\test|} \sum_{x, y \in \test} (y - \hat y)^{2},
  \quad \hat y = \E_{p_{\beta}(y | x, \train)}[y]
\end{equation*}
and accuracy for the classification examples
\begin{equation*}
  \dfrac{1}{|\test|} \sum_{x, y \in \test} \mathds{1}[y = \hat y],
  \quad \hat y = \argmax_{y} p_{\beta}(y | x, \train)
\end{equation*}
where $\mathds{1}$ is an indicator function. Note that the MSE and accuracy are
indifferent to the constructions of PPD~\eqref{eq:smpd} and~\eqref{eq:tmpd}, as
they both produce identical point predictions $\hat y$ in our setup.


\section{Model and Prior Details}
\label{sec:model-prior}
\subsection{Network Architecture and Datasets}
The network architecture and datasets are described in this section. The
dimension of $x$ and $\param$, and the sizes of the training, validation and
testing sets are presented in \cref{tab:data-model}.

\paragraph{One-layer ReLU network on UCI datasets.}
The scalar mean function $\mu(\cdot; \param)$ is parameterized with a 64-neuron
hidden layer using ReLU activations. The datasets are Concrete
\citep{yeh07concrete}, Naval \citep{coraddu14condition} and Energy
\citep{tsanas12energy}. They are all provided under the CC BY 4.0 license from
UC Irvine Machine Learning Repository.

\paragraph{CNN on MNIST.}
We utilize the convolutional neural network (CNN) provided in the MNIST example
within the Flax tutorial. In broad terms, it comprises two convolutional layers
(with 32 and 64 filters, respectively), followed by two fully connected layers
(with 256 and 10 outputs). The convolutional layers employ $3 \times 3$
convolutions with ReLU activations and $2 \times 2$ average pooling. The MNIST
dataset \citep{lecun10mnist} is provided under the CC BY-SA 3.0 license. Code
for the CNN model can be found
in \url{https://github.com/google/flax/blob/main/examples/mnist/train.py} under
the Apache License, Version 2.0.

\paragraph{ResNet20 on CIFAR10.}
We use the ResNet20 architecture \citep{he16deep} as implemented and ported from
\citet{wenzel20how}. We also use the following data augmentation scheme, as in
\citet{wenzel20how}:
\begin{itemize}
  \item Random left and right flipping, then;
  \item Border-padding 4 zero values in both horizontal and vertical direction,
        followed by random cropping of the image to its original size.
\end{itemize}
The CIFAR10 dataset \citep{krizhevsky09learning} can be found
in \url{https://www.cs.toronto.edu/~kriz/cifar.html}. Code for the ResNet20
model can be found
in
\url{https://github.com/google-research/google-research/blob/master/cold_posterior_bnn/models.py}
under the Apache License, Version 2.0.

\subsection{Error Variance of the Gaussian Regression Model}
As the training data $y$ in our experiment is standardized to unit variance, it
is safe to assume that error variance $\sigma^{2}$ of the trained model will be
less than 1. Therefore, we set $\sigma^{2} = 0.1$.

\subsection{Prior Variance}
\label{sec:prior-variance}
We use an isotropic Gaussian $\mathcal{N}(0, \sigma^{2}_{p})$ for the neural
network weights. The variance $\sigma^{2}_{p}$ is specified in
\cref{tab:data-model}.

\begin{table*}[t]
  \caption{Details of the datasets, size of neural network and prior variance.
    We use the same prior for both CIFAR10 experiments with or without data
    augmentation.}
  \label{tab:data-model}
  \vskip 0.15in
  \begin{center}
    \begin{small}
      \begin{sc}
        \begin{tabular}{lcccccc}
          \toprule
                                          & Concrete & Energy & Naval & MNIST                   & CIFAR10                 \\
          \midrule
          $\operatorname{dim}(x)$         & 8        & 8      & 14    & $28 \times 28 \times 1$ & $32 \times 32 \times 3$ \\
          $\operatorname{dim}(\param)$    & 641      & 641    & 1025  & 824458                  & 273258                  \\
          $|\train|$                      & 824      & 614    & 9547  & 60000                   & 50000                   \\
          $|\valid|$                      & 103      & 77     & 1194  & 5000                    & 5000                    \\
          $|\test|$                       & 103      & 77     & 1193  & 5000                    & 5000                    \\
          Prior variance $\sigma^{2}_{p}$ & 0.1      & 0.1    & 1     & 0.1                     & 1                       \\
          \bottomrule
        \end{tabular}
      \end{sc}
    \end{small}
  \end{center}
  \vskip -0.1in
\end{table*}

\section{SGMCMC}
\label{sec:sgmcmc}
We use the implementation of SGMCMC as presented in \citet{wenzel20how}. This
corresponds to the stochastic gradient Hamiltonian Monte Carlo
\citep{chen14stochastic} with rescaled hyperparameters and an adaptive scaling
on the Gaussian noise to ensure efficient sampling from a tempered posterior.
The samples of
$p_{\beta}(\param | \train) \propto \exp{-\beta U(\param)} = \exp { -U(\param) / \T }$,
where $U(\cdot)$ is the \textit{posterior energy function} defined as
\begin{equation}
  \label{eq:potential-energy}
  U(\param) \coloneq - \log p(\param) - \sum_{x, y \in \train} \log p(y | x, \param),
\end{equation}
can be drawn by simulating the following Langevin stochastic difference equation
(SDE) over $\param \in \mathbb{R}^{d}$ and momentum $\m \in \mathbb{R}^{d}$
\begin{align}
  \label{eq:sde-param}
  \textrm{d} \param & = \mathbf{M}^{-1}\, \m \, \textrm{d}t,                               \\
  \label{eq:sde-m}
  \textrm{d} \m     & = -\nabla_{\param} U(\param)\, \textrm{d}t -\gamma \m \, \textrm{d}t
  + \sqrt{2 \gamma \T} \mathbf{M}^{1/2} \textrm{d} \mathbf{W},
\end{align}
for any \textit{friction} $\gamma > 0$. Here, $\mathbf{W}$ is a Wiener process,
which can be loosely interpreted as a generalized Gaussian distribution
\citep{leimkuhler15molecular}. The \textit{mass matrix} $\mathbf{M}$ is a
preconditioner that can help in speeding up the convergence to the limiting
distribution of this SDE. We also prefer working with $\T$ instead of the
inverse temperature $\beta$ in the sampler to facilitate sampler diagnostics and
temperature ramp-up.

In practice, the gradient of $U(\param)$ is approximated by a minibatch gradient
estimator
\begin{equation*}
  \nabla_{\param} \tilde U(\param) \coloneq
  - \nabla_{\param} \log p(\param)
  - \dfrac{|\train|}{|\Dbatch|} \sum_{x, y \in \Dbatch} \nabla_{\param} \log p(y | x, \param),
\end{equation*}
where $\Dbatch$ denotes a minibatch, and $| \Dbatch |$ and $| \train |$ denote
the batch size and number of training samples respectively. The SDEs are then
solved numerically with a first-order symplectic Euler discretization scheme
using this minibatch gradient estimator, resulting in the following update
equations
\begin{align}
  \label{eq:sgmcmc-update-m}
  \m^{(t)}     & = (1 - \h \gamma) \m^{(t-1)} - \h \nabla_\param \tilde U(\param^{(t-1)}) + \sqrt{2 \gamma \h \T} \mathbf{M}^{1/2} \mathbf{R}^{(t)} \\
  \nonumber
  \param^{(t)} & = \param^{(t-1)} + h \mathbf{M}^{-1} \m^{(t)}
\end{align}
where $\mathbf{R}^{(t)} \sim \mathcal{N}(0, I_{d})$ is a standard Gaussian
vector. Note that the temperature shows up in the update equations and
effectively scales the random noise. This is helpful for drawing samples from
cold posteriors, which tend to have narrow, high density regions. The scaling
prevents the Markov chain from taking steps that are too large and missing the
high-density region.

The step size $h$ is often also modulated with a scheduler
$C(t): \mathbb{R}^{+} \to [0, 1]$
\begin{equation}
  \label{eq:step-size-scheduler}
  h = h_{0} C(t),
\end{equation}
where $h_{0}$ is the initial step size. Where appropriate, we also use the
layerwise preconditioner as proposed in \citet{wenzel20how} to speed up
convergence and reduce approximation error.

\section{SGMCMC Diagnostics}
\label{sec:sgmcmc-diagnostics}
In this work, we use the kinectic temperature diagnostic \citep{wenzel20how} to
assess the quality of SGMCMC samples. This is a departure from the common
practice of using test~LPD as a proxy for the posterior approximation error,
since this been shown to be a poor proxy \citep{deshpande22are}. In addition, we
also report the rank-normalized split-$\hat R$ statistics
\citep{vehtari21ranknormalization} on the (unnormalized) posterior density,
which are invariant to the permutation of neural network weights.

\subsection{Kinetic Temperature}
We report the expected kinetic temperature of each SGMCMC chain at different
temperatures, as proposed in \citet[Appendix I]{wenzel20how}. The kinetic
temperature estimator is given by
\begin{equation}
  \label{eq:ktemp}
  \hat \T(\m) = \dfrac{\m^{\top} \mathbf{M}^{-1} \m}{\operatorname{dim}(\m)}
\end{equation}
where $\m$ and $\mathbf{M}$ are the (random) momentum and the mass matrix in the
SDE \eqref{eq:sde-param}-\eqref{eq:sde-m}. For a perfect simulation of the SDE,
\eqref{eq:ktemp} is an unbiased estimator of the temperature of the system,
i.e.,~$\E[\hat \T(\m)] = \T$ \citep[Section
  6.1.5]{leimkuhler15molecular}.

In \cref{tab:ktemp}, the expected kinetic temperatures were computed by
averaging over the temperature samples over the whole Markov chain. We use these
estimates to gauge the simulation quality of the SDEs --- an estimate closer to
the target temperature indicates a better approximation of the SDEs. We
generally expect the simulation quality to worsen in the lower temperature
regime, due to the errors from both discretization and gradient sub-sampling
becoming more prominent relative to the randomness in $\m$. Otherwise, there is
no major concern about the sample quality, as observed in \cref{tab:ktemp}.

\begin{table*}[!t]
  \caption{The expected kinetic temperatures of each SGMCMC chain are presented
    here. In this table, the temperatures are inverted. Kinetic temperatures
    that are closer to the target indicate better simulation of the SDE. The
    kinetic temperatures in the $\hat \beta^{*}$ row differ across repetitions
    due to the variety of $\hat \beta^{*}$ obtained from SGD in each
    repetition.}
  \label{tab:ktemp}
  \vskip 0.15in
  \begin{center}
    \begin{small}
      \begin{sc}
        \begin{tabular}{clcccccc}
          \toprule
          Target $\beta$                    &        & Concrete & Energy & Naval  & MNIST  & CIFAR10 & CIFAR10 (DA) \\
          \midrule
          \multirow{5}{*}{0.1}              & Rep.~1 & 0.08     & 0.08   & 0.10   & 0.10   & 0.10    & 0.10         \\
                                            & Rep.~2 & 0.10     & 0.09   & 0.10   & 0.10   & 0.10    & 0.10         \\
                                            & Rep.~3 & 0.09     & 0.09   & 0.10   & 0.10   & 0.10    & 0.10         \\
                                            & Rep.~4 & 0.08     & 0.09   & 0.10   & 0.10   & 0.10    & 0.10         \\
                                            & Rep.~5 & 0.09     & 0.09   & 0.10   & 0.10   & 0.10    & 0.10         \\
          \midrule
          \multirow{5}{*}{0.3}              & Rep.~1 & 0.29     & 0.24   & 0.30   & 0.30   & 0.30    & 0.30         \\
                                            & Rep.~2 & 0.30     & 0.29   & 0.30   & 0.30   & 0.30    & 0.30         \\
                                            & Rep.~3 & 0.30     & 0.30   & 0.30   & 0.30   & 0.30    & 0.30         \\
                                            & Rep.~4 & 0.30     & 0.26   & 0.30   & 0.30   & 0.30    & 0.30         \\
                                            & Rep.~5 & 0.30     & 0.25   & 0.30   & 0.30   & 0.30    & 0.30         \\
          \midrule
          \multirow{5}{*}{1}                & Rep.~1 & 1.00     & 1.00   & 1.00   & 1.00   & 1.00    & 1.00         \\
                                            & Rep.~2 & 0.99     & 1.00   & 0.98   & 1.00   & 1.00    & 1.00         \\
                                            & Rep.~3 & 0.99     & 0.99   & 1.00   & 1.00   & 1.00    & 1.00         \\
                                            & Rep.~4 & 0.99     & 1.00   & 0.99   & 1.00   & 1.00    & 1.00         \\
                                            & Rep.~5 & 1.00     & 0.99   & 0.99   & 1.00   & 1.00    & 1.00         \\
          \midrule
          \multirow{5}{*}{3}                & Rep.~1 & 2.98     & 2.97   & 2.98   & 3.00   & 2.99    & 2.99         \\
                                            & Rep.~2 & 2.95     & 3.00   & 2.95   & 3.00   & 2.99    & 2.99         \\
                                            & Rep.~3 & 2.98     & 2.99   & 2.98   & 3.00   & 2.99    & 2.99         \\
                                            & Rep.~4 & 2.98     & 2.99   & 2.95   & 3.00   & 2.99    & 2.99         \\
                                            & Rep.~5 & 2.96     & 2.98   & 2.96   & 3.00   & 2.99    & 2.99         \\
          \midrule
          \multirow{5}{*}{10}               & Rep.~1 & 9.87     & 9.94   & 9.82   & 9.99   & 9.96    & 9.92         \\
                                            & Rep.~2 & 9.88     & 10.00  & 9.69   & 9.99   & 9.96    & 9.93         \\
                                            & Rep.~3 & 9.87     & 9.97   & 9.89   & 9.99   & 9.96    & 9.92         \\
                                            & Rep.~4 & 9.88     & 9.97   & 9.68   & 9.99   & 9.96    & 9.93         \\
                                            & Rep.~5 & 9.90     & 9.97   & 9.78   & 9.99   & 9.96    & 9.93         \\
          \midrule
          \multirow{5}{*}{30}               & Rep.~1 & 29.61    & 29.80  & 28.48  & 29.98  & 29.77   & 29.50        \\
                                            & Rep.~2 & 29.28    & 29.85  & 28.64  & 29.98  & 29.77   & 29.59        \\
                                            & Rep.~3 & 29.39    & 29.76  & 28.73  & 29.96  & 29.77   & 29.54        \\
                                            & Rep.~4 & 28.93    & 29.91  & 27.61  & 29.96  & 29.77   & 29.55        \\
                                            & Rep.~5 & 29.04    & 29.85  & 28.73  & 29.97  & 29.77   & 29.63        \\
          \midrule
          \multirow{5}{*}{100}              & Rep.~1 & 95.81    & 99.07  & 84.76  & 99.87  & 98.21   & 96.34        \\
                                            & Rep.~2 & 92.08    & 98.88  & 85.24  & 99.88  & 98.20   & 97.12        \\
                                            & Rep.~3 & 92.25    & 97.94  & 85.66  & 99.83  & 98.19   & 96.58        \\
                                            & Rep.~4 & 92.20    & 99.26  & 72.74  & 99.81  & 98.15   & 96.69        \\
                                            & Rep.~5 & 91.28    & 98.51  & 87.84  & 99.86  & 98.21   & 97.34        \\
          \midrule
          \multirow{5}{*}{300}              & Rep.~1 & 259.87   & 293.28 & 196.04 & 299.19 & 285.58  & 276.00       \\
                                            & Rep.~2 & 248.86   & 293.13 & 207.08 & 299.19 & 285.32  & 280.62       \\
                                            & Rep.~3 & 228.33   & 287.56 & 187.88 & 299.06 & 285.15  & 276.42       \\
                                            & Rep.~4 & 221.00   & 289.86 & 133.78 & 298.99 & 285.40  & 276.79       \\
                                            & Rep.~5 & 229.55   & 293.08 & 213.83 & 299.14 & 285.24  & 281.91       \\
          \midrule
          \multirow{5}{*}{1000}             & Rep.~1 & 547.10   & 939.37 & 346.19 & 992.38 & 860.32  & 775.48       \\
                                            & Rep.~2 & 649.51   & 943.01 & 399.87 & 992.25 & 857.79  & 821.73       \\
                                            & Rep.~3 & 503.20   & 875.31 & 286.99 & 991.97 & 853.45  & 792.92       \\
                                            & Rep.~4 & 436.40   & 880.59 & 202.85 & 991.61 & 853.14  & 798.56       \\
                                            & Rep.~5 & 446.14   & 958.23 & 432.83 & 991.98 & 855.56  & 844.81       \\
          \midrule
          \multirow{5}{*}{$\hat \beta^{*}$} & Rep.~1 & 0.16     & 7.44   & 47.53  & 2.65   & 21.04   & 7.28         \\
                                            & Rep.~2 & 0.17     & 4.86   & 44.12  & 2.68   & 3.49    & 13.93        \\
                                            & Rep.~3 & 0.11     & 7.91   & 33.33  & 4.13   & 3.44    & 13.72        \\
                                            & Rep.~4 & 0.08     & 4.12   & 42.88  & 3.58   & 3.46    & 14.17        \\
                                            & Rep.~5 & 0.10     & 4.58   & 27.26  & 2.03   & 20.48   & 14.28        \\
          \bottomrule
        \end{tabular}
      \end{sc}
    \end{small}
  \end{center}
  \vskip -0.1in
\end{table*}

\subsection{Rank-normalized Split-$\hat R$}
In \cref{tab:rhat}, we report the rank-normalized, split-$\hat R$
\citep{vehtari21ranknormalization} on the potential
energy~\eqref{eq:potential-energy}, which is invariant to the permutation of
neural network weights. This statistic is a strict improvement over the
traditional potential scale reduction factor $\hat R$ \citep{gelman92inference}
and split-$\hat R$ \citep{gelman13bayesian}. The $\hat R$ statistic compares
between-chain and within-chain variances, thus requiring multiple
independently-initialized Markov chains to compute. It was designed based on the
idea that, for Markov chains that are mixing well, they should converge to the
same limiting distribution regardless of initialization. Therefore, the
between-chain and within-chain variances should roughly be the same, and a
$\hat R$ closer to 1 is considered better. The term `split' indicates that the
Markov chains are divided in half, resulting in double the number of Markov
chains with half the original length. The variances are computed across these
doubled number of Markov chains to detect poor convergence within an original
chain. The `traditional' $\hat R$ is computed on the original values of the
potential energy, while the `rank-normalized' $\hat R$ is computed on the
rank-normalized values of the potential energy and does not require the limiting
distribution to have a finite mean or variance.

In Bayesian deep learning, the posterior distribution is almost always
multi-modal \citep{izmailov21what}, and we should not expect the
independently-initialized Markov chains to always converge to the same limiting
distribution. Therefore, we should not be overly alarmed by the relatively
`large' readings of $\hat R$ in \cref{tab:rhat}. These values are large by the
usual standard in Bayesian statistics but are common in Bayesian deep learning
\citep{izmailov21what,fortuin22bayesian}.

\begin{table*}[b]
  \caption{The split-$\hat R$ values for the log-posterior at different
    (inverse) temperatures. These split-$\hat R$ values are computed from 5
    SGMCMC chains initialized at different $\param$. Values closer to 1 indicate
    better mixing of the Markov chains. $\hat \beta^{*}$ is left out as the
    temperature (and hence the posterior) differs across repetitions.}
  \label{tab:rhat}
  \vskip 0.15in
  \begin{center}
    \begin{small}
      \begin{sc}
        \begin{tabular}{ccccccc}
          \toprule
          $\beta$ & Concrete & Energy & Naval & MNIST & CIFAR10 & CIFAR10 (DA) \\
          \midrule
          0.1     & 1.26     & 1.19   & 1.82  & 1.07  & 1.01    & 1.04         \\
          0.3     & 1.04     & 1.16   & 1.83  & 1.03  & 1.02    & 1.08         \\
          1       & 1.54     & 1.09   & 1.94  & 1.02  & 1.13    & 1.07         \\
          3       & 2.09     & 1.10   & 2.24  & 1.03  & 1.17    & 1.46         \\
          10      & 2.79     & 1.30   & 2.78  & 1.14  & 1.69    & 1.99         \\
          30      & 2.68     & 1.77   & 3.07  & 1.27  & 1.42    & 2.30         \\
          100     & 2.83     & 2.38   & 3.25  & 1.27  & 1.24    & 2.29         \\
          300     & 2.81     & 2.40   & 3.10  & 1.19  & 1.19    & 2.19         \\
          1000    & 3.97     & 2.87   & 3.20  & 1.16  & 1.15    & 2.38         \\
          \bottomrule
        \end{tabular}
      \end{sc}
    \end{small}
  \end{center}
  \vskip -0.1in
\end{table*}

\section{Hyperparameters and the Compute Environment}
\label{sec:hyperparameters}
In this section, we provide a detailed explanation of the hyperparameters for
SGD and SGMCMC. The values are summarized in \cref{tab:hyperparameters}. Our
source code is available on
\url{https://github.com/weiyaw/tempered-posteriors}.

\paragraph{Learning rate and scheduler (SGD).}
This is the learning rate scheduler. The cosine schedule starts from the
indicated learning rate and gradually decreases to 0 throughout the entire SGD
run. The `piecewise' scheduler is used in \citet{wenzel20how}, where the initial
learning rate is multiplied by a value at specified epochs. We write it in the
format of (epoch, multiplier): (80, 0.1), (120, 0.01), (160, 0.001), (180,
0.0005).

\paragraph{Weight decay.}
We subtract a weight decay term from the transformed gradient before scaling the
transformed gradient with a learning rate. This subtraction is performed because
we are maximizing the likelihood rather than minimizing a loss.

\paragraph{Gradient clipping.}
We clip the gradient norm to a specified threshold, as implemented in Optax
\citep{deepmind20deepmind}. This helps stabilize the training procedure.

\paragraph{Learning rate and momentum (SGMCMC).}
The learning rate and momentum terms will control both the initial step size
$h_{0}$ and the friction $\gamma$. The relationship between them is given in
\citet[Appendix B]{wenzel20how} and is repeated here:
\begin{equation*}
  h_{0} = \sqrt{\textit{learning rate} / \ntrain}, \quad \gamma = (1 - \textit{momentum}) / h_{0},
\end{equation*}
where $\ntrain$ is the size of the training set.

\paragraph{Scheduler and cycle length (SGMCMC).}
This is the scheduler $C(t)$ in \eqref{eq:step-size-scheduler} that modulates
$h$. During the burn-in phase, the scheduler is fixed at 1, i.e.~$h=h_{0}$.
Then, during the sampling phase, the step size is modulated with a cyclical
cosine schedule \citep{zhang20cyclical} with a period of the specified epoch.
One sample is collected at the end of each cycle.

\paragraph{Ramp start and end.}
As we run our SGMCMC algorithm from $\T = 0$ and only start increasing $\T$
after a specified number of epochs, these two hyperparameters indicate the epoch
when $\T$ was gradually increased and the epoch when $\T$ reaches the target
temperature.

\paragraph{Burn-in epochs, total epochs and usable samples (SGMCMC).}
The `total epochs' indicate the total number of epochs run by SGMCMC, including
epochs during the burn-in phase. `Usable samples' indicate the number of samples
collected after the burn-in phase.

\begin{table*}[t]
  \caption{Hyperparameters for SGD and SGMCMC. We use the same set of
    hyperparameters for both CIFAR-10 experiments, with or without data
    augmentation.}
  \label{tab:hyperparameters}
  \vskip 0.15in
  \begin{center}
    \begin{small}
      \begin{sc}
        \begin{tabular}{clccccc}
          \toprule
                                                     &                    & Concrete  & Energy    & Naval     & MNIST     & CIFAR10   \\
          \midrule
          \multirow{8}{*}{\rotatebox[]{90}{SGD}}     & Learning rate      & $10^{-6}$ & $10^{-6}$ & $10^{-8}$ & $10^{-6}$ & $10^{-6}$ \\
                                                     & Scheduler          & Cosine    & Cosine    & Cosine    & Cosine    & Piecewise \\
                                                     & Momentum           & 0.9       & 0.9       & 0.9       & 0.9       & 0.9       \\
                                                     & Nesterov           & No        & No        & No        & No        & No        \\
                                                     & Weight decay       & 1         & 1         & 1         & 1         & 500       \\
                                                     & Batch size         & Full      & Full      & 128       & 128       & 128       \\
                                                     & Total epochs       & 15000     & 15000     & 10000     & 10        & 200       \\
                                                     & Gradient clipping  & $10^{6}$  & $10^{6}$  & $10^{4}$  & $10^{6}$  & $10^{6}$  \\
          \midrule
          \multirow{12}{*}{\rotatebox[]{90}{SGMCMC}} & Learning rate      & $10^{-3}$ & $10^{-3}$ & $10^{-4}$ & 0.01      & 0.1       \\
                                                     & Scheduler          & Cyclical  & Cyclical  & Cyclical  & Cyclical  & Cyclical  \\
                                                     & Cycle length       & 200       & 200       & 100       & 10        & 50        \\
                                                     & Preconditioner     & None      & None      & None      & None      & Layerwise \\
                                                     & Momentum           & 0.98      & 0.98      & 0.98      & 0.98      & 0.98      \\
                                                     & Batch size         & Full      & Full      & 128       & 128       & 128       \\
                                                     & Ramp start (epoch) & 4800      & 4800      & 900       & 10        & 100       \\
                                                     & Ramp end (epoch)   & 5000      & 5000      & 1000      & 20        & 150       \\
                                                     & Burn-in epochs     & 10000     & 10000     & 5000      & 200       & 500       \\
                                                     & Total epochs       & 30000     & 30000     & 15000     & 1200      & 2000      \\
                                                     & Usable samples     & 100       & 100       & 100       & 100       & 30        \\
                                                     & Gradient clipping  & $10^{6}$  & $10^{6}$  & $10^{6}$  & $10^{6}$  & $10^{6}$  \\
          \bottomrule
        \end{tabular}
      \end{sc}
    \end{small}
  \end{center}
  \vskip -0.1in
\end{table*}

\subsection{Computational Environment and Resources}
\label{sec:comp-envir}
The algorithms were implemented in \texttt{JAX} \citep{bradbury18jax}. We used
\texttt{Flax} \citep{heek23flax} and \texttt{Tensorflow Probabilities}
\citep{dillon17tensorflow} to implement neural networks and models. Plots were
generated with \texttt{tidyverse} \citep{wickham19welcome} and \texttt{ggplot2}
\citep{wickham16ggplot2}. Rank-normalized split-$\hat R$ was computed with the
\texttt{posterior} package \citep{burkner23posterior}. We also used
\texttt{xarray} \citep{hoyer17xarray} and \texttt{arviz} \citep{kumar19arviz} to
conduct explanatory analysis, and \texttt{wandb} \citep{biewald20experiment} to
track our experiments. Our code will be released.

The experiment was conducted on Google Cloud Platform utilizing TPU-V3. In each
machine, there are 8 TPU cores with 16GB of TPU memory attached to each core
(i.e.,~totaling 128GB of memory in each machine). However, we only utilize one
TPU core due to the difficulty in parallelizing the MCMC chains across different
cores, and effectively only have access to 16GB of memory in each machine.

On a CIFAR10 experiment, a SGD run (for our proposed algorithm) takes roughly
0.5 hour to complete. A SGMCMC run at a particular temperature takes roughly 4.5
hours. In our main experiment, as we are computing over a grid of 9 temperature
plus the optimal temperature, one repetition of the experiment takes 45.5 hours.
Multiplying this by five repetitions and we get 227.5 hours to produce one of
the CIFAR10 subplot in \cref{fig:test-tempered-lpd-compare}. The rest of the
subplot are considerably cheaper: 16 hours for MNIST, and less than an hour for
each of Concrete, Naval and Energy.

In total, it takes 474 hours to generate the posterior samples for computing
\cref{tab:test-results} and \cref{fig:test-tempered-lpd-compare}. It takes an
additional 279 hours for the experiment that varies data augmentation strength,
i.e.,~\cref{fig:test-lpd-da}. Hyperparamters tuning takes an additional 200
hours. We have also spent roughly 100 compute-hours to try out various prior for
$\beta$.

\subsection{Standard Error Calculation}
\label{sec:stand-error-calc}
The standard error (SE) of the means are computed with \texttt{sd} routine in
\texttt{R}. More specifically, it computes the square root of an unbiased
estimate of the variance. The upper and lower bound of the shaded areas in all
figures are `mean $+$ SE' and `mean $-$ SE' respectively.

\section{Extra Figures and Tables From the Main Experiment}
\label{sec:extra-figures}
In this section, we provide additional figures and tables to complement the
results presented in the main text.

We first show the validation LPD plotted against the temperature in
\cref{fig:val-lpd} and find that our method can recover temperatures in the
high-performing regions. We further improve the visualization of
\cref{fig:test-tempered-lpd-compare} by separating the test LPD computed with
SM-PD \eqref{eq:smpd} and TM-PD \eqref{eq:tmpd} into \cref{fig:test-lpd} and
\cref{fig:test-tempered-lpd}, respectively. A tabular version of the results is
also provided in \cref{tab:lpd-regression} and \cref{tab:lpd-classification},
with an SGD reference included in all the results.

We also show the performance of point predictions across different temperatures
in \cref{fig:test-pe}, complementing \cref{tab:test-results}. Note that the
point predictions are identical for SM-PD and TM-PD. We observe that predictions
from tempered posteriors generally outperform SGD.

Finally, we show the wall-clock time comparison across the methods in
\cref{tab:time}. Overall, our method is 4 times faster than the grid search for
the smaller regression models and 8 times faster for the larger classification
models.

\begin{table*}[t]
  \caption{Test LPD of the regression models. The values presented are means
    $\pm$ standard errors across five repetitions, with the best value among the
    four methods highlighted in bold. TM-PDs generally have better performance
    for the regression models. Higher LPD indicates better performance.}
  \label{tab:lpd-regression}
  \vskip 0.15in
  \begin{center}
    \begin{small}
      \begin{sc}
        \begin{tabular}{llccc}
          \toprule
                                                           & Method                   & Concrete                    & Energy                     & Naval                      \\
          \midrule
          \multirow{4}{*}{\rotatebox[origin=c]{90}{SM-PD}} & SGD                      & $-3.904 \pm 1.000$          & $1.299 \pm 0.011$          & $1.376 \pm 0.003$          \\
                                                           & $\beta = 1$              & $-0.506 \pm 0.255$          & $1.221 \pm 0.005$          & $1.376 \pm 0.000$          \\
                                                           & $\beta = \hat \beta^{*}$ & $\mathbf{-0.186 \pm 0.049}$ & $1.287 \pm 0.009$          & $\mathbf{1.382 \pm 0.000}$ \\
                                                           & Grid                     & $-0.216 \pm 0.121$          & $\mathbf{1.312 \pm 0.003}$ & $\mathbf{1.382 \pm 0.001}$ \\
          \midrule
          \multirow{4}{*}{\rotatebox[origin=c]{90}{TM-PD}} & SGD                      & $-0.319 \pm 0.076$          & $1.75 \pm 0.054$           & $2.95 \pm 0.162$           \\
                                                           & $\beta = 1$              & $-0.506 \pm 0.255$          & $1.22 \pm 0.005$           & $1.38 \pm 0.000$           \\
                                                           & $\beta = \hat \beta^{*}$ & $-0.13 \pm 0.069$           & $1.78 \pm 0.063$           & $3.16 \pm 0.106$           \\
                                                           & Grid                     & $\mathbf{0.019 \pm 0.035}$  & $\mathbf{1.85 \pm 0.017}$  & $\mathbf{3.71 \pm 0.194}$  \\
          \bottomrule
        \end{tabular}
      \end{sc}
    \end{small}
  \end{center}
  \vskip -0.1in
\end{table*}

\begin{table*}[t]
  \caption{Test LPD of the classification models. The presented values are the
    means $\pm$ standard error across five repetitions, with the best value
    among the four methods boldfaced. SM-PDs tend to perform better on CIFAR-10,
    while SM-PD and TM-PD have similar performance on MNIST. Higher LPD
    indicates better performance.}
  \label{tab:lpd-classification}
  \vskip 0.15in
  \begin{center}
    \begin{small}
      \begin{sc}
        \begin{tabular}{llccc}
          \toprule
                                                           & Method                   & MNIST                       & CIFAR10                     & CIFAR10 (DA)                \\
          \midrule
          \multirow{4}{*}{\rotatebox[origin=c]{90}{SM-PD}} & SGD                      & $-0.076 \pm 0.032$          & $-1.138 \pm 0.306$          & $-1.234 \pm 0.176$          \\
                                                           & $\beta = 1$              & $-0.021 \pm 0.000$          & $-0.333 \pm 0.001$          & $-0.343 \pm 0.002$          \\
                                                           & $\beta = \hat \beta^{*}$ & $\mathbf{-0.018 \pm 0.000}$ & $\mathbf{-0.317 \pm 0.006}$ & $-0.232 \pm 0.007$          \\
                                                           & Grid                     & $\mathbf{-0.018 \pm 0.000}$ & $-0.319 \pm 0.003$          & $\mathbf{-0.229 \pm 0.004}$ \\
          \midrule
          \multirow{4}{*}{\rotatebox[origin=c]{90}{TM-PD}} & SGD                      & $-0.030 \pm 0.003$          & $-0.55 \pm 0.036$           & $-0.284 \pm 0.004$          \\
                                                           & $\beta = 1$              & $-0.021 \pm 0.000$          & $\mathbf{-0.333 \pm 0.001}$ & $-0.343 \pm 0.002$          \\
                                                           & $\beta = \hat \beta^{*}$ & $\mathbf{-0.017 \pm 0.000}$ & $-0.719 \pm 0.443$          & $-0.811 \pm 0.221$          \\
                                                           & Grid                     & $\mathbf{-0.017 \pm 0.001}$ & $\mathbf{-0.333 \pm 0.001}$ & $\mathbf{-0.279 \pm 0.007}$ \\
          \bottomrule
        \end{tabular}
      \end{sc}
    \end{small}
  \end{center}
  \vskip -0.1in
\end{table*}

\begin{figure*}
  \vskip 0.2in
  \begin{center}
    \centerline{\includegraphics[width=\linewidth]{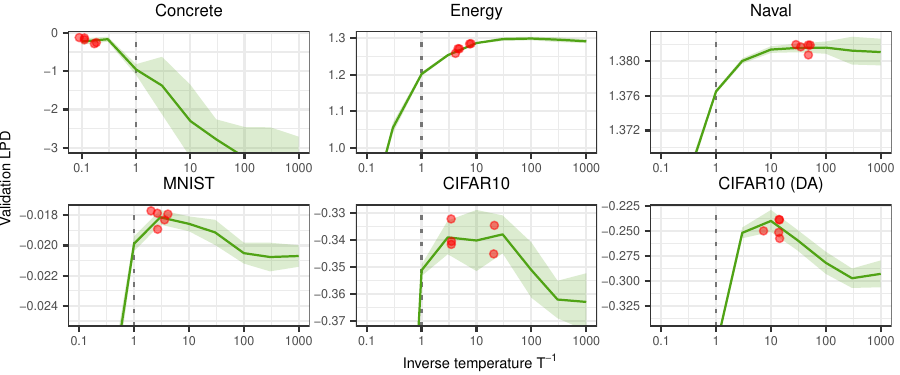}}
    \caption{Validation LPD plotted against inverse temperature $\beta$. This is
      computed with SM-PD as defined in \eqref{eq:smpd}. Solid lines and shaded
      area represent mean $\pm$ standard error across five repetitions. The
      vertical dotted lines indicate the PPD at $\beta = 1$. There are five red
      dots in each plot, each of them corresponding to a repetition of
      $\hat \beta^{*}$ from our method. Higher LPD indicates better
      performance.}
    \label{fig:val-lpd}
  \end{center}
  \vskip -0.2in
\end{figure*}

\begin{figure*}
  \vskip 0.2in
  \begin{center}
    \centerline{\includegraphics[width=\linewidth]{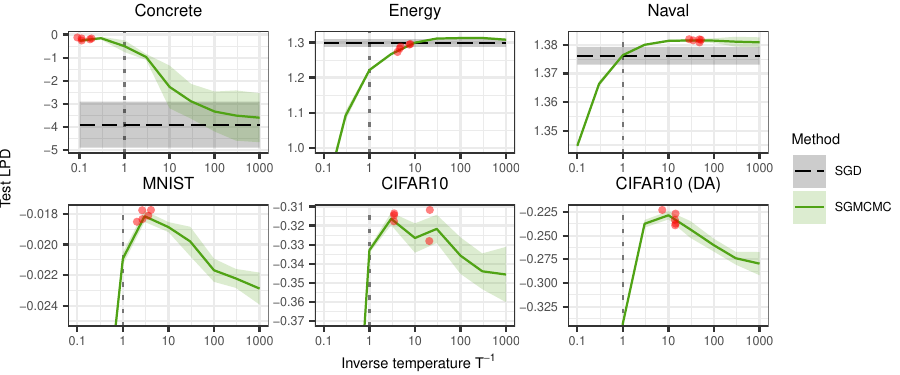}}
    \caption{Test LPD plotted against inverse temperature $\beta$ with SGMCMC
      (green, solid). This is computed with SM-PD as defined in \eqref{eq:smpd}.
      The SGD solution (horizontal, black, dashed) is included as a reference.
      The SGD reference in MNIST and CIFAR10 examples performs considerably
      worse and is out of range. Lines and shaded area represent mean $\pm$
      standard error across five repetitions. The vertical dotted lines indicate
      the PPD at $\beta = 1$. There are five red dots in each plot, each of them
      corresponding to a repetition of $\hat \beta^{*}$ from our method. Higher
      LPD indicates better performance.}
    \label{fig:test-lpd}
  \end{center}
  \vskip -0.2in
\end{figure*}

\begin{figure*}
  \vskip 0.2in
  \begin{center}
    \centerline{\includegraphics[width=\linewidth]{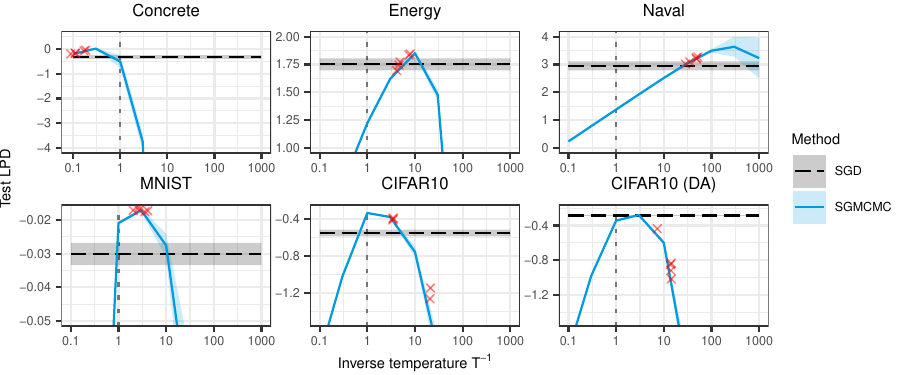}}
    \caption{Test LPD plotted against inverse temperature $\beta$ with SGMCMC
      (blue, solid). This is computed with TM-PD as defined in \eqref{eq:tmpd}.
      The SGD solution (horizontal, black, dashed) is included as a reference.
      Lines and shaded area represent mean $\pm$ standard error across five
      repetitions. The vertical dotted lines indicate the PPD at $\beta = 1$.
      There are five red dots in each plot, each of them corresponding to a
      repetition of $\hat \beta^{*}$ from our method. Higher LPD indicates
      better performance.}
    \label{fig:test-tempered-lpd}
  \end{center}
  \vskip -0.2in
\end{figure*}

\begin{figure*}
  \vskip 0.2in
  \begin{center}
    \centerline{\includegraphics[width=\linewidth]{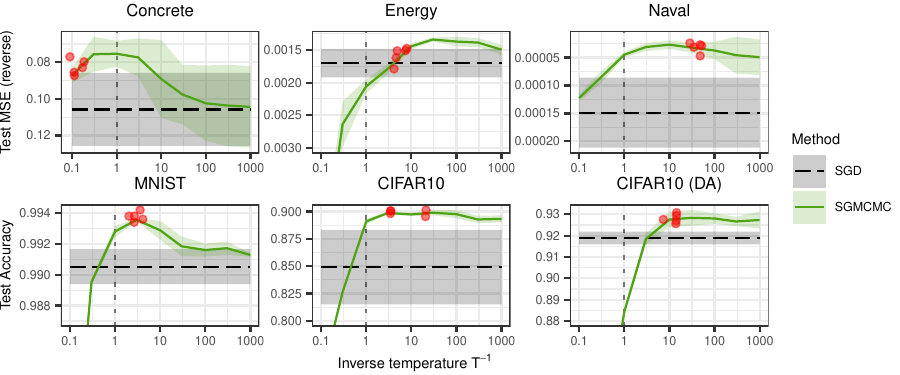}}
    \caption{Test~MSE (Concrete, Energy, Naval) and accuracy (MNIST, CIFAR10) of
      the point predictions of the PPDs, plotted against inverse temperature
      $\beta$. The results for SM-PD \eqref{eq:smpd} and TM-PD \eqref{eq:tmpd}
      are consolidated, as both produce identical point predictions by
      definition. The SGD solution (horizontal, black, dashed) is included as a
      reference. Lines and shaded area represent mean $\pm$ standard error
      across five repetitions. The vertical dotted lines indicate the PPD at
      $\beta = 1$. There are five red dots in each plot, each of them
      corresponding to a repetition of $\hat \beta^{*}$ from our method. Note
      the MSE has been reversed, and higher value indicates better performance
      in all plots.}
    \label{fig:test-pe}
  \end{center}
  \vskip -0.2in
\end{figure*}

\begin{table*}[t]
  \caption{Wall-clock time comparison across methods. Times are averaged over
    five repetitions.}
  \label{tab:time}
  \vskip 0.15in
  \begin{center}
    \begin{small}
      \begin{sc}
        \begin{tabular}{lcccccc}
          \toprule
          \multirow{2}{*}{Method}         & \multicolumn{6}{c}{Wall-clock time (minutes)}                                                    \\
          \cmidrule(lr){2-7}
                                          & Concrete                                      & Energy & Naval & MNIST  & CIFAR10 & CIFAR10 (DA) \\
          \midrule
          SGD                             & 0.17                                          & 0.16   & 0.78  & 1.93   & 27.90   & 34.84        \\
          $\beta = 1$                     & 0.20                                          & 0.19   & 1.40  & 18.83  & 260.52  & 327.12       \\
          $\beta = \hat \beta^{*}$ (Ours) & 0.37                                          & 0.36   & 2.19  & 20.77  & 288.42  & 361.96       \\
          Grid                            & 1.76                                          & 1.74   & 12.64 & 169.51 & 2344.15 & 2944.01      \\
          \bottomrule
        \end{tabular}
      \end{sc}
    \end{small}
  \end{center}
  \vskip -0.1in
\end{table*}

\section{The Decision-Theoretic Generalized Bayesian Framework}
\label{sec:bissiri-primer}
The classical Bayesian rule~\eqref{eq:standard-rule} updating prior belief
$p(\param)$ to the posterior belief $p(\param | \train)$ arises from Bayes'
Theorem. An implicit assumption in this procedure for optimal performance is
that the model is correctly specified: there exists some $\theta_0$ such that
$p(y \vert x, \theta_0) = q(y|x)$ for all $(x,y)$. The adage that ``all models
are wrong'' \citep{box76science} highlights that this assumption should not be
expected to hold. Consequently, it does not make sense to infer on a parameter
$\param$ that lacks a connection to the true data-generation distribution.

Using a decision-theoretic argument, \citet{bissiri16general} offers an
alternative justification for using~\eqref{eq:standard-rule}, even when the
model is misspecified. Suppose that we are interested in a quantity $\param_{0}$
that minimizes the expected loss under the truth~$q(x, y)$, that is,
\begin{equation}
  \label{eq:argmin-loss}
  \param_{0} \coloneq \argmin_{\param} \int \ell(\param, x, y) q(x, y) \, \textrm{d}y \, \textrm{d}x
\end{equation}
for some loss function $\ell$. Then, we would like to derive an update rule
$\psi$ that takes an observation $(x_{1}, y_{1})$ and updates our prior belief
on $\param_{0}$ to a posterior belief,
\begin{equation*}
  p(\param | x_{1}, y_{1}) = \psi \{ \ell(\param, x_{1}, y_{1}), p(\param) \}.
\end{equation*}
Furthermore, the update rule should be \emph{coherent}. That is, suppose that we
would like to update our posterior with two data points
$\{ (x_{1}, y_{1}), (x_{2}, y_{2}) \}$, the update rule should satisfy
\begin{equation*}
  \psi \{ \ell(\param, x_{2}, y_{2}), \psi \{ \ell(\param, x_{1}, y_{1}), p(\param) \} \}
  \equiv \psi \{ \ell(\param, x_{2}, y_{2}) + \ell(\param, x_{1}, y_{1}), p(\param) \}.
\end{equation*}
The coherent property will ensure that the same posterior arises regardless of
the order which the data are processed. \citet{bissiri16general} shows that all
coherent update rules take the form
\begin{equation}
  \label{eq:general-rule}
  p_{\operatorname{GB}}(\param | \train) \propto p(\param)
  \exp \left( -\sum_{(x,y) \in \train} \ell(\param, x, y) \right) ,
\end{equation}
which we will refer to as the \textit{general Bayesian update}, and the
associated posterior $p_{\operatorname{GB}}(\param | \train)$ as the
\textit{generalized Bayes posterior}, which has been widely analyzed
\citep{zhang06epsilon, zhang06informationtheoretic, jiang08gibbs,
  bhattacharya19bayesian, alquier16properties, alquier20concentration}. Note
that we must take care to have a finite normalizing constant to
\eqref{eq:general-rule},
\begin{equation*}
  0 < \int \exp \left( - \sum_{(x,y)\in \train }\ell(\param, x, y) \right) p(\param) \textrm{d}\param < +\infty.
\end{equation*}
Under certain regularity conditions, the generalized Bayes posterior will
concentrate on $\param_{0}$ as \mbox{$\ntrain \to \infty$}; a proof can be found
in \citet[Lemma~5]{mclatchie24predictive}. The key advantage of this framework
in lies in its ability to directly infer the quantity of interest
(i.e.,~$\param_{0}$) using a loss function, without necessitating any
probabilistic assumptions that relate~$\param_{0}$ to the data. For formal
proofs, we direct interested readers to \citep[Section 1]{bissiri16general}.

\subsection{Calibrating Information Gain}
The tempered posterior, as defined in~\eqref{eq:temper-posterior}, is a special
parameterisation of the generalized Bayes posterior achieved by setting the loss
to be $-\log p(y | x, \param)$ and scaling it by a factor of $\beta > 0$,
i.e.~$\ell(\param, x, y) = -\beta \log p(y | x, \param)$, and concentrating the
prior to an equivalent extent. Therefore, the tempered posterior has an
interpretation of a coherent update rule for the KL minimizer
$\param_{\dag} = \argmin_{\param} -\beta \log p(y | x, \param) = \argmin_{\param} \KL(q(y | x) \| p(y | x, \param))$,
as defined in the main text. From this perspective, the temperature has the role
of controlling the amount of information being `added' to the prior in each
update step. \citet{holmes17assigning} proposed setting $\beta$ such that the
\textit{expected information gain} of our setup (LHS
of~\eqref{eq:expected-information-gain}) matches that of a hypothetical
experiment (RHS of~\eqref{eq:expected-information-gain}),
\begin{equation}
  \label{eq:expected-information-gain}
  \int D(p_{\beta}(\param | \mathcal{D}), p(\param)) q(x, y) \mathrm{d}y \mathrm{d}x
  = \int D(p(\param | \mathcal{D}), p(\param)) p(y | x, \param_\dag) q(x) \mathrm{d}y \mathrm{d}x
\end{equation}
where $D$ is a divergence that measures the information gain from a prior to a
posterior, and $p_{\beta}$ is a tempered posterior (with likelihood tempering
only).

On the LHS of~\eqref{eq:expected-information-gain}, we are measuring the
expected information gain in our existing setup, i.e.,~$q(x, y)$ is unknown,
from a prior to a tempered posterior $p_{\beta}$. Note that this integral is a
function of $\beta$. Then, \citet{holmes17assigning} proposed picking a $\beta$
that matches up the expected information gain from a well-specified experiment
that targets the same quantity of interest $\param_{\dag}$, i.e.,~having the
truth be $p(\cdot | x, \param_\dag) q(x)$. In the well-specified setting, the
optimal $\beta$ is 1, and thus the information gain is computed from the prior
to a posterior at $\beta = 1$. This proposed method is based on the principle
that the expected information gain should match across posteriors that are
targeting the same quantity of interest. In their follow up work,
\citet{lyddon19general} proposed selecting $\beta$ such that the asymptotic
Fisher information number of the generalized Bayes posterior to match to that
derived from the ``loss-likelihood'' bootstrap method, which is a sampling
method for general Bayes posteriors.

\subsection{Data Augmentation in Generalized Bayes}
\label{sec:da-generalized-bayes}

Accounting for data augmentation in BDL is non-trivial, as the data are no
longer i.i.d. This has led to the development of several techniques to address
this issue \citep{nabarro22data,kapoor22uncertainty}. However, data augmentation
can be easily incorporated into decision-theoretic GB \citep{bissiri16general},
as the generalized Bayes posterior is derived from iterative belief updates
rather than relying on the independence assumption. We can then regard the
posterior as reflecting our belief in the KL minimizer from the `augmented'
truth to our model.

For example, in the ResNet20-CIFAR10 example with SGMCMC in \citet{wenzel20how},
the data was augmented in each epoch, and we effectively have $M \ntrain$
(instead of $\ntrain$) number of observations in the training set
$\tilde \train = \{ (\tilde x_{i}, y_{i}) \}^{M\ntrain}_{i=1}$, where $M$ is the
number of epochs, and $\tilde x_{i}$ represents an augmented image. Then,
setting $\ell(x, y, \param) = - \log p(y | x, \param)$ and working `backward',
the posterior can be regarded as representing our belief in a new quantity
$\tilde \param_{\dag}$. This quantity is the minimizer of the KL divergence from
an `augmented' truth $\tilde q(x, y) = \tilde q(y | x) \tilde q(x)$, from which
the augmented data have arisen, to the model $p(y | x, \param)$
\begin{equation*}
  \tilde \param_{\dag} = \argmin_{\param} \E_{\tilde q(x)} \KL(\tilde q(y | x) \| p(y | x, \param)) \, .
\end{equation*}

\end{document}